\documentclass[Afour,sageh,times]{sagej}

\usepackage{moreverb,url}

\usepackage[colorlinks,bookmarksopen,bookmarksnumbered,citecolor=red,urlcolor=red]{hyperref}
\usepackage{amsmath,amsfonts,amssymb,bm,amsthm,amssymb}
\usepackage{mathrsfs}
\usepackage{algorithmic}
\usepackage{algorithm}
\usepackage{array}
\usepackage{subfigure}
\usepackage{textcomp}
\usepackage{stfloats}
\usepackage{url}
\usepackage{verbatim}
\usepackage{graphicx}
\usepackage{appendix}
\usepackage{multirow}
\theoremstyle{definition}
\newtheorem{Def}{Definition}
\newtheorem{remark}{Remark}
\newtheorem{prop}{Proposition}
\newtheorem{exmp}{Example}
\newtheorem{lemma}{Lemma}
\newtheorem{theorem}{Theorem}

\newcommand\BibTeX{{\rmfamily B\kern-.05em \textsc{i\kern-.025em b}\kern-.08em
T\kern-.1667em\lower.7ex\hbox{E}\kern-.125emX}}

\def\volumeyear{2022}

\begin{document}

\runninghead{Mao and Quan}

\title{Optimal Virtual Tube Planning and Control for Swarm Robotics}

\author{Pengda Mao\affilnum{1}, Rao Fu\affilnum{1} and Quan Quan\affilnum{1}}

\affiliation{\affilnum{1}  School of Automation Science and Electrical Engineering, Beihang University, Beijing, P.R. China}

\corrauth{Quan Quan, School of Automation Science and Electrical Engineering,
  Beihang University, Beijing, 100191, P.R. China.}

\email{ qq\_buaa@buaa.edu.cn}

\begin{abstract}
  This paper presents a novel method for efficiently solving a trajectory planning problem for swarm robotics in cluttered environments.
  Recent research has demonstrated high success rates in real-time local trajectory planning for swarm robotics in cluttered environments, but optimizing trajectories for each robot is still computationally expensive, with a computational complexity from $O\left(k\left(n_t,\varepsilon \right)n_t^2\right)$ to $ O\left(k\left(n_t,\varepsilon \right)n_t^3\right)$ where $n_t$ is the number of parameters in the parameterized trajectory, $\varepsilon$ is precision and $k\left(n_t,\varepsilon \right)$ is the number of iterations with respect to $n_t$ and $\varepsilon$. Furthermore, the swarm is difficult to move as a group.
  {To address this issue, we define and then construct the \emph{optimal virtual tube}, which includes infinite optimal trajectories.} Under certain conditions, any optimal trajectory in the optimal virtual tube can be expressed as a convex combination of a finite number of optimal trajectories, with a computational complexity of $O\left(n_t\right)$.
  {Afterward, a hierarchical approach including a planning method of \emph{the optimal virtual tube} with minimizing energy and distributed model predictive control is proposed.}
  In simulations and experiments, the proposed approach is validated and its effectiveness over other methods is demonstrated through comparison.
\end{abstract}

\keywords{Swarm robotics, Trajectory planning, Virtual tubes, Optimization}

\maketitle

\section{Introduction}
Swarm robotics has potential applications in various real-world scenarios, such as air traffic control, land and sea search and rescue, and target detection. 
A certain amount of research has been dedicated to the smooth and safe navigation of swarm robotics passing through obstacle-dense environments.

{
Passing through an obstacle-dense environment safely and smoothly as a group is still a challenging task for swarm robotics. Fast and smooth movement of a swarm can be achieved through distributed trajectory planning, but such method may also cause the swarm to split apart. On the other hand, centralized formation path planning combined with multi-robot formation control can achieve swarm passing-through in formation, but it lacks flexibility and maneuverability. In this work, an optimal virtual tube planning method inspired by the application of tubes in transporting objects in real life is proposed for a robotic swarm to centrally plan an optimal virtual tube that provides a safe area without obstacles and forward commands in the environment; then the distributed model predictive control is used to track the commands, thereby achieving fast and smooth passing of the swarm as a group in obstacle-dense environments.
}
\begin{figure}
	\centering
	\includegraphics[width=0.99\linewidth]{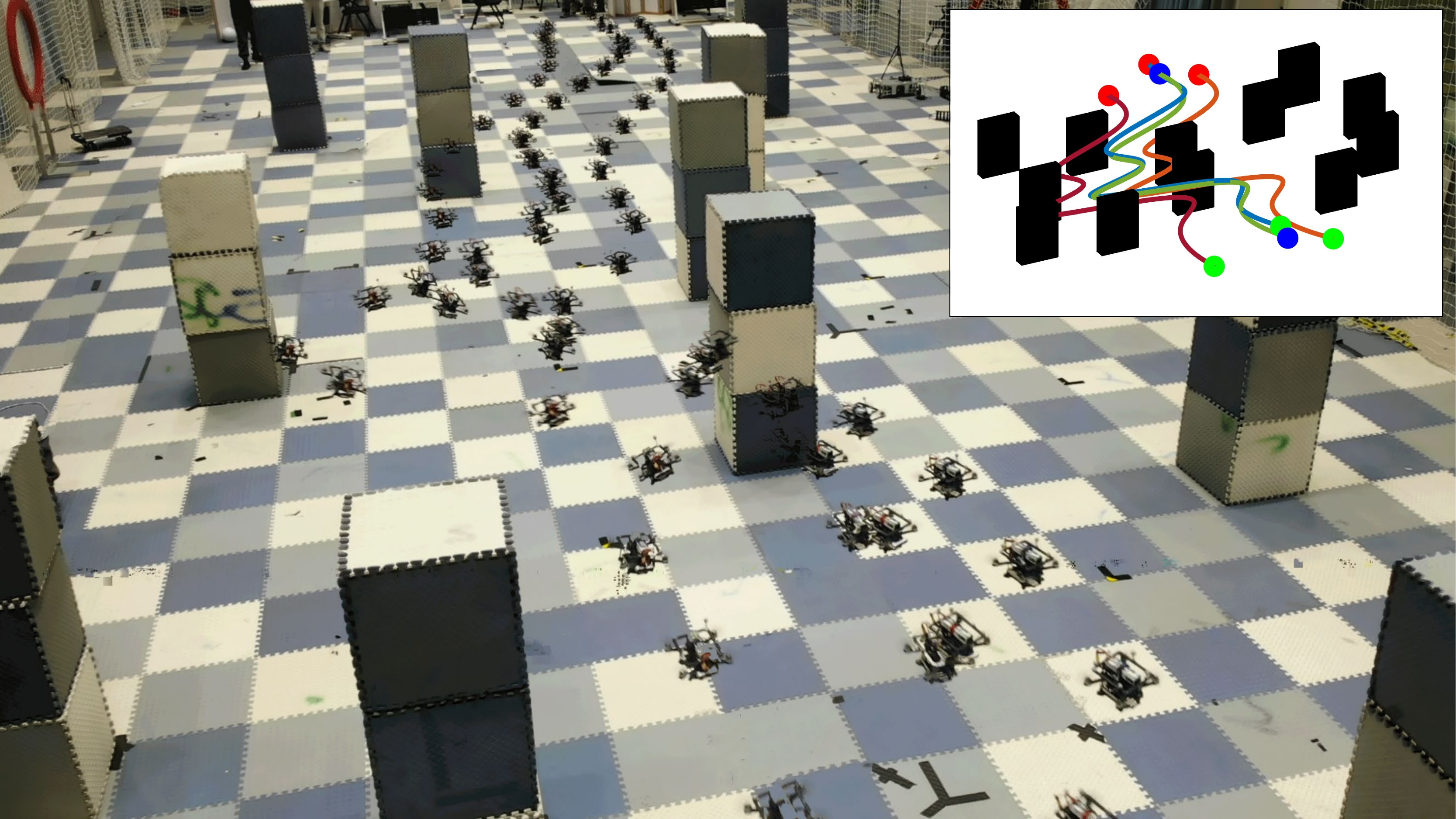}
	\caption{By optimal virtual tube method, an overlay image of four drones passing through the cluttered environment: The trajectory of each drone is labeled by a different color in the upper right figure. Compared to the traditional method, our method solves only three optimization problems for the three optimal trajectories of drones, such as the red, blue, and orange lines. The fourth optimal trajectory of the drone (the green line) is generated by an affine combination.}
	\label{fig:fourdronesoverlap}
\end{figure}

{
We employ the path planning algorithm to plan a safe area without obstacles for a swarm to form a group, which effectively simplifies the obstacle avoidance problem by replacing obstacle constraints with the tube boundaries. Hence, the robots in the swarm only need to consider inter-robot collision avoidance and tube boundary constraints. Furthermore, the fact that optimal trajectories do not intersect with each other in the tube substantially reduces the chance that collision avoidance control is triggered. Based on optimization theories, we propose a method with a computational complexity of $O\left(n_t\right)$, where $n_t$ is the number of parameters in trajectory to generate infinite optimal trajectories with minimizing energy. This method has a significant computational advantage over existing methods with the computational complexity from $O\left(k\left(n_t,\varepsilon \right)n_t^2\right)$ to $O\left(k\left(n_t,\varepsilon \right)n_t^3\right)$ where $\varepsilon$ is precision and $k\left(n_t,\varepsilon \right)$ is the number of iterations needed to solve the optimization problem, making it possible to generate optimal trajectories for large-scale robot swarms in a timely and centered fashion.
The optimal virtual tube planning method is validated effectively by using the distributed model predictive control method in simulations and experiments. Furthermore, the superiority of our method is demonstrated through comparison with other methods.
}

\subsection{Related work}
In this section, we provide an overview of the related studies on topics including trajectory planning, infinite homotopic trajectory generation, and virtual tube control and planning. 
\subsubsection{Trajectory planning}
Trajectory planning is an efficient method for robots to pass through a complex environment. Considering the kinematics of the robot and obstacles, an optimization problem is formulated to plan a smooth and safe trajectory for each robot.
By leveraging the differential flatness, the optimal trajectory could be tracked by some robots.
The methods of trajectory parameterization mainly include $k$-th order polynomials and B-spline.
(i) To ensure passing through fixed waypoints, a trajectory can be parameterized as $k$-th order polynomials (\cite{mellinger_minimum_2011}). However, trajectories described as polynomials are prone to collisions with obstacles due to finite discrete avoidance constraints. To address this issue, various methods have been proposed for generating collision-free trajectories. For instance, a collision-free trajectory has been generated using mixed integer methods in \cite{deits2015efficient}.
An approach proposed by \cite{inaba_polynomial_2016} refines the trajectory iteratively by minimizing a cost function and adding additional waypoints between two ends of a particular trajectory segment that passes through obstacles to avoid collisions. 
However, these methods are computationally intensive for real-time planning.
To overcome these challenges, some methods have been proposed that use simple constraints to design a quadratic programming (QP) problem, which can significantly reduce the computation time. For example, safe flight corridor (SFC), a collection of convex connected polyhedra that models free space in a map, is used as a constraint for obstacle avoidance to generate a collision-free trajectory in real-time (\cite{Liu2017Planning,gao_teach-repeat-replan_2020}).
(ii) The trajectory parameterized by B-spline (\cite{usenko_real-time_2017}) is proposed to eliminate constraints of intermediate conditions on any derivative of the trajectory. An essential advantage of using B-splines is the convex hull property, which imposes bounds on both the trajectory and its derivatives. 
This property allows for strict bounding of the position, velocity, and acceleration of a robot (\cite{ding_efficient_2019}). 
The drawback of B-spline, however, is that the B-spline curve does not pass through the control points that represent waypoints. 
To repel the trajectory from obstacles, a gradient-based method is applied (\cite{zhou_robust_2020}) such that a penalty function is introduced. Although this method has been successful, it is experience dependent and lacks rigorous theoretical proof. It has been shown that if the control points satisfy certain conditions, the trajectory will strictly avoid obstacles (\cite{zhou2019robust}). 

To plan swarm trajectories, existing methods are mostly extension of single trajectory planning methods. However, in swarm robotics, other agents are regarded as dynamic obstacles, which increase the complexity of the problem. 
{Several methods have been proposed to solve this avoidance problem, such as RSFC (Relative Safe Flight Corridor) (\cite{Park2021Online}), distributed predictive control (\cite{soria_distributed_2022}), HOOP (Hold Or take Optimal Plan) (\cite{tang2018hold}) and EGO (ESDF-free Gradient-based lOcal) Swarm (\cite{zhou_ego-swarm_2021}), all of which employ increasingly complex constraints.} Safe and smooth trajectories need to be generated by optimization problems, which require significant computational resources with an increasing scale of the swarm. Furthermore, the constraint of gathering the swarm as a group is also needed to increase the computation cost.
Therefore, the method efficiency for swarm robotics is a crucial point needed to be considered, especially with the limited onboard computing resources of robots.

\subsubsection{Infinite homotopic trajectories generation}The objective of trajectory planning for a robot passing through an environment is to plan a trajectory from a start position to a goal position for each robot.  Ideally, trajectory planning would yield a unique solution to the optimization problem, but \cite{orthey2019motion} demonstrates that local minima in the optimization problem result in an infinite set of homotopic collision-free trajectories. 
To address this challenge, a generative model of collision-free trajectory is trained in \cite{osa2022motion} to represent an infinite set of homotopic solutions to the optimization problem. Motivated by the methods of finding homotopic trajectories for a robot, we proposed an \emph{optimal virtual tube} for swarm robotics, which contains an infinite number of homotopic optimal trajectories in free space.
Our focus is on finding infinite trajectories for a swarm, rather than for a single robot, although generating infinite optimal trajectories through infinite optimization problems is not feasible due to computational complexity.  
To mitigate this issue, we first propose and define the optimal virtual tube, including infinite optimal trajectories. {Then, a theorem is proposed so that any optimal trajectory in the virtual tube could be represented by a convex combination of these finite number of optimal trajectories with a computational complexity of $O\left(n_t\right)$ if trajectories are linearly parameterized and satisfy the conditions in lemmas.} Therefore, combining the advantages of formation control (\cite{alonso2017multi}), the complexity of the optimal virtual tube planning is independent of the scale of the swarm\footnote{The scale of a swarm means the number of robots in the swarm.}.

\subsubsection{Virtual tube control and planning}
{
The ``virtual tube" provides safety \emph{boundary} and \emph{direction} of motion for a swarm. This concept appeared in AIRBUS' SkyRoad project, is a free space for unmanned aerial vehicles to fly over cities (\cite{airbus}). These virtual tubes establish a safe flight area in which drone flights will not interfere with ground or air traffic. 
Motivated by this, in our previous work (\cite{Quan2021Practical}), a \emph{straight virtual tube} was proposed for swarm robotics. There are no obstacles inside the virtual tube, so the area inside can be regarded as a safety zone. Restricting the robot to fly in the safety zone can simplify the control, that is, only requiring to guarantee no collision with other robots and the no collision among robots and the boundary of the virtual tube, rather than the collision avoidance with complex obstacles in environment. The virtual tube ensures the safety of robots and cities.
Then, to guide the swarm to pass through corridors, doors and windows, and surround surveillance targets without diverging (\cite{gao2022multi}), we have also generalized the definition of virtual tubes and proposed the \emph{curved virtual tube} (\cite{Quan2021Distributed}).
This concept of the curve virtual tube is similar to the lane for autonomous road vehicles in \cite{rasekhipour2016potential}, \cite{luo2018porca} and corridor for a multi-UAV system in \cite{nagrare2022multi}.
}

{
As a summary, the earlier part of this subsection is dedicated to the introduction of two problems, namely the \emph{virtual tube planning problem} and \emph{virtual tube passing-through control problem}. However, The above-mentioned methods all apply to passing-through control problems.
For the virtual tube planning problem, in previous work (\cite{Mao2022}), a generator curve is first obtained by trajectory planning based on several discrete waypoints, which are generated using search-based methods. Then the virtual tube is generated by expanding the generator curve while avoiding obstacles. 
}
\subsection{Motivations and contributions}
{
Although there are many studies on the virtual tube passing-through control problem, a suitable design of the forward commands for the swarm is still lacking. And the previous virtual tube planning work also has many limitations. The motivations of the optimal virtual tube are demonstrated as follows.}

{
First, the kinematics of the robot are not considered in design of the forward commands for the swarm in the virtual tube passing-through control problem. The desired speeds for robots are manually designed and the directions of the forward velocity commands of the robots parallel to the direction of the generator curve, which may cause the swarm to block performing turning and other maneuvers.}

{
Secondly, the previous virtual tube planning method (\cite{Mao2022}) poses limitations to the control algorithm and is not scalable in high-dimension spaces. The regular virtual tube is only suitable for some specific methods such as APF (Artificial Potential Field), flocking algorithm, and CBF (Control Barrier Function). And it is difficult to be extended into high-dimension spaces.}

{
Thirdly, trajectories of swarm robotics in existing research are not optimal, namely any robot in the virtual tube is not assigned an optimal trajectory.
On the other hand, it costs much computation with the number of robots increasing if the trajectory optimization method is adopted.
}

{
To overcome these limitations, this paper first generalizes the definition of the virtual tube without a \emph{generator curve} so that it is convenient to be applied in a high-dimension space, and the optimal virtual tube which includes infinite optimal trajectories is proposed for trajectory planning of swarm robotics. An optimal virtual tube planning method is proposed to generate the infinite optimal trajectories in the optimal virtual tube with a small computational complexity which makes the virtual tube applicable to a wider range of control methods.
Then the optimal virtual tube planning method combined with model predictive control is validated in simulations and experiments, as shown in Fig. \ref{fig:fourdronesoverlap}. The main contributions are as follows.
}
{
\begin{itemize}
	\item Compared with our previous work (\cite{Mao2022}), the definition of the virtual tube is extended into a high-dimension space. Further, the optimal virtual tube is defined for trajectory planning of swarm robotics, whose properties and advantages are analyzed in this paper. There are infinite optimal trajectories within the optimal virtual tube. Each robot is assigned to an optimal trajectory without intersecting with other trajectories.
	\item Based on the definitions and theorems, we show that a linear virtual tube with convex hull terminals is a type of optimal virtual tubes which is easy to implement in practice. The infinite optimal trajectories are generated by the convex combination of finite optimal trajectories. Therefore, the computational complexity is only $O\left(n_t\right)$, in contrast with solving existing optimization problems with computational complexity of $O\left(k\left(n_t,\varepsilon \right)n_t^2\right)$ to $O\left(k\left(n_t,\varepsilon \right)n_t^3\right)$.  
	\item A hierarchical approach including an optimal virtual tube planning method and model predictive control is proposed. The corresponding application on drone swarm is demonstrated in simulations and experiments. Furthermore, comparisons with other methods are implemented in simulations to demonstrate the superior performance of our approach.
\end{itemize} 
}
\subsection{Paper organization}
{
This paper is organized as follows. 
The preliminaries about topology and convex optimization, and problem formulation are described in Section 2.
Section 3 is devoted to demonstrating the definitions such as virtual tube, optimal virtual tube, and linear virtual tube, and distinguishes the proposed work from our previous work.
Then, a specific virtual tube, named linear virtual tube with convex hull terminals, is defined and proved to be an optimal virtual tube based on a proposed theorem in Section 4 .
In Section 5, a planning method to obtain an optimal virtual tube is proposed.
Model predictive control is proposed for the swarm to track trajectories in an optimal virtual tube as an validation of the tube planning method in Section 6. 
Experimental results and comparisons in simulations are presented to illustrate the superior performance of our approach in Section 7.
Finally, Section 8 contains the conclusion and future work. 
}
\section{Preliminaries and Problem Formulation}
\subsection{Topology}
The topological theory is complex and profound. For the purpose of comprehensibility, we will roughly and intuitively introduce some basic concepts that need to be used in the following sections.
{
\subsubsection{Interior} The \emph{interior} of the subset $\mathcal{T}$ of a topological space, denoted by ${\bf int}{\rm{\,}}{\mathcal{T}}$, is the largest open subset contained in $\mathcal{T}$, expressed as
\[{\bf{int}}\left( {\cal T} \right) = \left\{ {{\bf{x}} \in {\cal T}|B\left( {{\bf{x}},r} \right) \cap {\cal T} \subseteq {\cal T}{\mkern 1mu} {\rm{\, for \, some\, }} r > 0} \right\},\]
where $B\left({\bf x},r\right) = \{ {\bf y} | \left\| {{\bf y} - {\bf x}} \right\| \le r\}$ represents the ball of radius $r$ and center ${\bf x}$ in the any norm $\left\|  \cdot  \right\|$.}
\subsubsection{Diffeomorphism}
Intuitively, a manifold (\cite{tu_introduction_2011}) is a generalization of curves and surfaces to higher dimensions.
Given two manifolds $M$ and $N$, a differentiable map $f:M \to N$ is called a \emph{diffeomorphism} if it is a bijection and its inverse is differentiable as well. 
\subsubsection{Simple connectivity (\cite{Shick_topology_2007})}
Intuitively, a space is \emph{simply-connected} if every loop in the space can be shrunk to a point within the space.

\subsection{Convex optimization (\cite{Boyd_convex_2004})}
\subsubsection{Convex hull}
For a set ${\mathcal{C}}=\{{\bf{x}}_i\},i=1,2,...,q$, the \emph{convex hull} of the set ${\mathcal{C}}$, denoted ${\bf{conv}} \, {\mathcal{C}}$, is the set of all convex combinations of points in ${\mathcal{C}}$:
\[{\bf{conv}}{\rm{\, }}{\mathcal{C}} = \left\{ {\sum\limits_{i = 1}^q {{\theta _i}{{\bf{x}}_i}} |{{\bf{x}}_i} \in {\mathcal{C}},{\theta _i} \ge 0,\sum\limits_{i = 1}^q {{\theta _i}}  = 1} \right\}.\]
{
The ${\bf{int}}\left( {{\bf{conv}}{\rm{\, }}{\cal C}} \right)$ and $\partial \left( {{\bf{conv}}{\rm{\, }}{\cal C}} \right)$ denote the \emph{interior} of ${\bf{conv}}{\,}{\mathcal{C}}$ and the \emph{boundary} of the ${\bf conv}{\rm{\,}}\mathcal{C}$ respectively.
}
An example is shown in Fig. \ref{fig:convexhull2d}.
\begin{figure}
	\centering
	\includegraphics[width=0.7\linewidth]{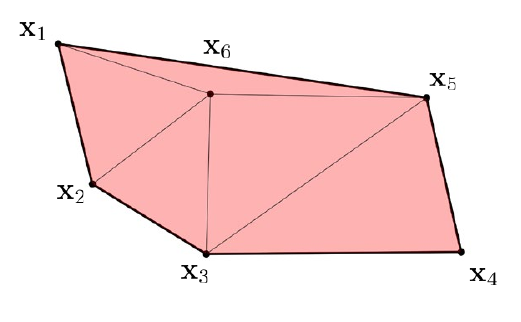}
	\caption{{Convex hull example. For a set ${\mathcal{C}} = \{{\bf x}_i\}$ ($ i=1,2,...,6$), the pink area is the ${\bf{conv}}{\rm{\, }}{\mathcal{C}}$. The $\partial \left( {{\bf{conv}}{\rm{\, }}{\cal C}} \right)$ and ${\bf{int}}\left( {{\bf{conv}}{\rm{\, }}{\cal C}} \right)$ denote the bold black lines and the pink area without the bold black lines respectively.}}
	\label{fig:convexhull2d}
\end{figure}
\subsubsection{An optimality criterion for differentiable $f_0$}
Suppose that a convex function $f_0$ is differentiable, so that for all ${\bf x,y}\in \text{\bf{dom}} f_0$ (domain of the function $f_0$),
\begin{equation}
{f_0}\left( {\bf y} \right) \ge {f_0}\left( {\bf x} \right) + \nabla {f_0}{\left( {\bf x} \right)^{\rm{T}}}\left( {\bf y - x} \right).
\end{equation}
Let ${\mathcal X}_f$ denote the feasible set. Then $\bf x$ is optimal if and only if ${\bf x} \in {\mathcal X}_f$ and 
\begin{equation}
\nabla {f_0}{\left( {\bf x} \right)^{\rm{T}}}\left( {\bf y - x} \right) \ge 0, \text{for\,all\,} {\bf y} \in {\mathcal X}_f.
\end{equation}

\subsection{{Problem formulation}}
{
Let $\mathcal{X}_c$ be the configuration space and denote the obstacle-free space as $\mathcal{X}_{\rm free}$. The trajectory between two given points ${\bf q}_0, {\bf q}_m \in {\mathcal{X}_{\rm free}}$ is defined by a map ${\bf{h}}\left( {\left( {{{\bf{q}}_0},{{\bf{q}}_m}} \right),t} \right)$ where $t \in \left[0,1\right]$, ${\bf{h}}\left( {\left( {{{\bf{q}}_0},{{\bf{q}}_m}} \right),0} \right) = {{\bf{q}}_0}$, ${\bf{h}}\left( {\left( {{{\bf{q}}_0},{{\bf{q}}_m}} \right),1} \right) = {{\bf{q}}_m}$. 
\begin{Def}
	A trajectory ${\bf{h}}^*\left( {\left( {{{\bf{q}}_0},{{\bf{q}}_m}} \right),t} \right)$ is \emph{optimal} with respect to cost $g$ if 
	\[ {{\bf h}^*}\left( {\left( {{{\bf{q}}_0},{{\bf{q}}_m}} \right),t} \right)  {\in} \arg \mathop{ \min }\limits_{{\bf h}\in \mathcal{H}} g\left( {{\bf h}\left( {\left( {{\bf{q}}_0,{\bf{q}}_m} \right),t} \right)} \right)\]
	where $t \in \left[0,1\right]$ and $\mathcal{H}$ is a candidate space, $g$ could be the energy cost, length of trajectory or other concerned. {Especially when the cost $g$ is convex, the optimal trajectory ${\bf{h}}^*\left( {\left( {{{\bf{q}}_0},{{\bf{q}}_m}} \right),t} \right)$ is unique, namely} $\color{blue} {{\bf h}^*}\left( {\left( {{{\bf{q}}_0},{{\bf{q}}_m}} \right),t} \right)  = \arg \mathop{ \min }\limits_{{\bf h}\in \mathcal{H}} g\left( {{\bf h}\left( {\left( {{\bf{q}}_0,{\bf{q}}_m} \right),t} \right)} \right).$
	\label{def:optimal-trajectory}
\end{Def} 
Given a start point ${\bf q}_0$ and a goal point ${\bf q}_m$, the trajectory planning problem for a single robot is to find an optimal trajectory ${{\bf{h}}^*}\left( {\left( {{{\bf{q}}_0},{{\bf{q}}_m}} \right),t} \right)\in \mathcal{X}_{\rm free}$ from ${\bf q}_0$ to ${\bf q}_m$. Consequently, given start points $\{{\bf q}_{0,i}\}$ for robots respectively in a start area $\mathcal{C}_0$ and a goal area $\mathcal{C}_1$, the trajectory planning problem for swarm robotics in an obstacle-dense environment is to assign goal points ${\bf q}_{m,i}$ for all robots and then find optimal homotopic trajectories with minimal energy ${{\bf{h}}^*}\left( {\left( {{{\bf{q}}_{0,i}},{{\bf{q}}_{m,i}}} \right),t} \right)$ from all start points ${\bf q}_{0,i}\in\mathcal{C}_0$ to the goal points ${\bf q}_{m,i} \in \mathcal{C}_1$. And then, a control method is to be adopted to make the robot swarm move fast and smoothly as a group.}

{
However, as the number of robots in a swarm increases, the computational cost of generating optimal trajectories also increases. 
Therefore, we use the optimal virtual tube to constrain the safety space where trajectories are located in and realize the generation of optimal trajectories with a low computational complexity of $O\left(n_t\right)$ in the following sections. }

\section{Optimal Virtual Tube}
A virtual tube is constructed for swarm robotics, which could be used for swarm robotics flow. A generalization of a virtual tube is defined as a set in space, and some specific types of the virtual tube including the optimal virtual tube are introduced for trajectory planning of swarm robotics. 

Tubes are often used for fluid transportation in three-dimension Euclidean space, which is denoted by $\mathbb{R}^3$. However, it is not sufficient to describe a motion of a robot in three-dimension Euclidean space. For example, a set of all affine motions, including rotations and translations in ${\mathbb{R}^3}$, is a six-dimension manifold. But this six-dimension manifold is not ${\mathbb{R}^6}$. Since the virtual tube aims at swarm robotics, it is extended to be defined as a set in $n$-dimension space.


Intuitively, a virtual tube is constructed by defining maps between two bounded convex sets in space. The convex set is simply connected so that there are no ``holes'' in the set. The mathematical definition of the virtual tube is formulated in the following.
\begin{Def}
	A \emph{virtual tube} $\mathcal T$, as shown in Fig. \ref{fig:tubedef}, is a set in $n$-dimension space represented by a 4-tuple $\left( {\mathcal C}_0, {\mathcal C}_1, {\bf f}, {\bf h} \right)$ where
	\begin{itemize}
		\item ${\mathcal C}_0, {\mathcal C}_1$, called \emph{terminals}, are disjoint bounded convex subsets in $n$-dimension space.
		\item ${\bf f}$ is a \emph{diffeomorphism}: ${\mathcal C}_0 \to {\mathcal C}_1$, so that there is a set of order pairs ${\mathcal{P}} = \left\{ {\left( {{\bf q}_0,{\bf q}_m} \right)| {\bf q}_0 \in {\mathcal{C}_0},{\bf q}_m = {\bf f}\left( {\bf q}_0 \right) \in {\mathcal{C}}_1} \right\}$.
		\item ${\bf h}$ is a smooth\footnote{A real-valued function is said to be smooth if its derivatives of all orders exist and are continuous.} map: ${\mathcal P} \times {\mathcal I} \to {\mathcal T} $ where ${\mathcal I} = [0,1]$, such that ${\mathcal T} = \{ {\bf h} \left( \left({\bf q}_0,{\bf q}_m\right),t\right) | \left({\bf q}_0,{\bf q}_m\right) \in {\mathcal P} , t \in {\mathcal I} \}$, ${\bf h}\left( {\left( {\bf q}_0,{\bf q}_m \right),0} \right) = {\bf q}_0$, ${\bf h}\left( {\left( {\bf q}_0,{\bf q}_m \right),1} \right) = {\bf q}_m$. The function ${\bf h}\left( {\left( {\bf q}_0,{\bf q}_m \right),t} \right)$ is called a \emph{trajectory} for an order pair $\left( {\bf q}_0,{\bf q}_m\right)$.
	\end{itemize}
	And, a \emph{cross-section} ${\mathcal{C}}_t$ of a virtual tube at $t \in {\mathcal{I}}$ is expressed as:
	\begin{equation}
	{\mathcal{C}}_{t} = \left\{ {{\bf h}\left( {\left( {\bf q}_0,{\bf q}_m \right),{t}} \right)|\left({\bf q}_0,{\bf q}_m\right) \in {\mathcal{P}} } \right\}.
	\end{equation}
	The surface of the virtual tube is the boundary of ${\mathcal{T}}$, defined as $\partial {\mathcal{T}}$.
	\label{def:virtual-tube}
\end{Def}

In the following, some properties are directly obtained based on \emph{Definition 1}.
\begin{prop}
	All cross-sections of a virtual tube are simply connected.
\end{prop}
\begin{proof}
	The convex sets ${\mathcal{C}_0, \mathcal{C}_1}$ are simply connected. Thus, the order pair ${\mathcal{P}}$ is simply connected. For a cross-section ${\mathcal{C}}_{t_0}$, there is a smooth map ${\bf h}$ from ${\mathcal{P}\times t_0}$ to ${\mathcal{C}}_{t_0}$. Therefore, the cross-section is simply connected. $\hfill\blacksquare$
\end{proof}
\begin{prop}
	All cross-sections of a virtual tube are continuous. That is, for every point in the cross-section, there exists at least one trajectory through it. 
\end{prop}
\begin{proof}
	The domain $\mathcal{P} \times \mathcal{I}$ is continuous. And the map ${\bf h}$ is continuous. So all cross-sections are continuous. $\hfill\blacksquare$
\end{proof}

\begin{figure}
	\centering
	\includegraphics{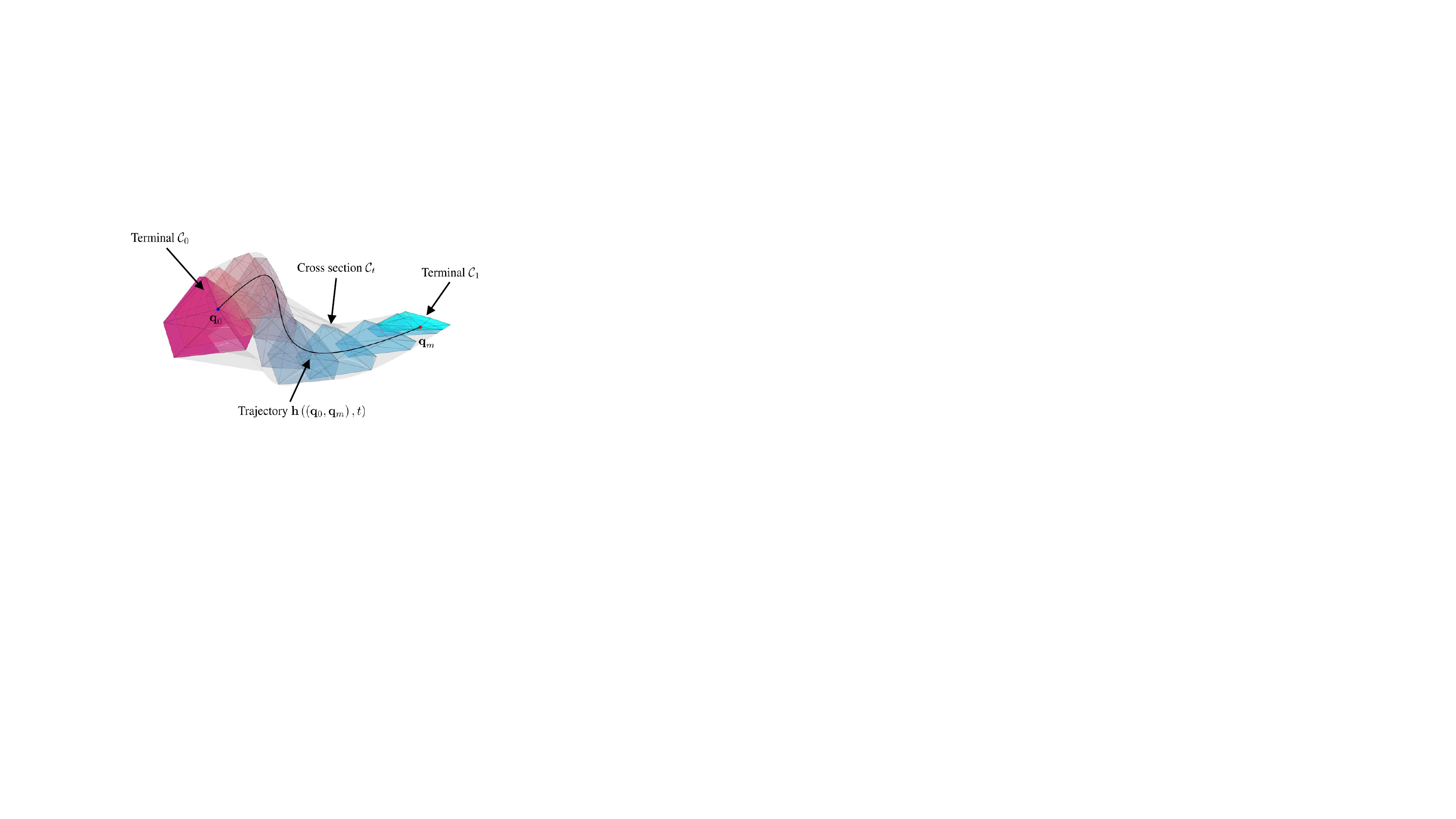}
	\caption{Virtual tube example. The purple and blue polyhedrons are terminals. The shaded polyhedrons are cross sections ${\mathcal C}_t$. The black curve is a trajectory that is from ${\bf q}_{0}$ in terminal $\mathcal{C}_0$ to ${\bf q}_{m}$ in terminal $\mathcal{C}_1$. The gray area is the virtual tube. }
	\label{fig:tubedef}
\end{figure}

{
\begin{remark}
	It should be noted that there is no explicit definition of a generator curve in \emph{Definition \ref{def:virtual-tube}}, compared with the definition of the virtual tube with a generator curve in our previous work (\cite{Mao2022}). In the following, we will show that the previous definition is a special case of \emph{Definition \ref{def:virtual-tube}}. The virtual tube with a generator curve is defined as 
	\[{\mathcal{T}}\left( {s,\theta ,\rho } \right) = {\bm{\gamma }}\left( s \right) + \rho \lambda \left( {s,\theta } \right)\left( {{\bf{n}}\left( s \right)\cos \theta  + {\bf{b}}\left( s \right)\sin \theta } \right),\] 	where  $s \in \mathcal{D}_s= \left[ {{s_0},{s_f}} \right] \subset  \mathbb{R} $, $\rho  \in {\mathcal{D_\rho }} = \left[ {{\rho _{\min }}\left( s \right),1} \right]$, $\theta  \in \mathcal{D}_{\theta}  \subset \mathbb{R}$, $\rho _{\min }\left(s\right)$ is a real-value function with respect to $s$, the second order differentiable curve ${\bm{\gamma }}\left( s \right) \in \mathbb{R}^n$ 
	is called \emph{generator curve}, ${\bf{n}}\left( s \right) = \frac{{\ddot {\bm \gamma} \left( s \right)}}{{\left\| {\dot {\bm \gamma} \left( s \right)} \right\|}} - \frac{{\dot {\bm \gamma} \left( s \right) \cdot \ddot {\bm \gamma} \left( s \right)}}{{{{\left\| {\dot {\bm \gamma} \left( s \right)} \right\|}^3}}}\dot {\bm \gamma} \left( s \right)$ and ${\bf{b}}\left( s \right) = \frac{{{\bm{\dot \gamma }}\left( s \right)}}{{\left| {{\bm{\dot \gamma }}\left( s \right)} \right|}} \times {\bf{n}}\left( s \right)$ $\in \mathbb{R}^n$ are \emph{principle normal} and \emph{binormal} vectors of generator curve $\bm{\gamma}$ at the point $\bm{\gamma}\left( s \right)$ respectively, and $\lambda \left( {s,\theta } \right) \in    \mathbb{R}$ called \emph{radius} is continuous with respect to $s$. 
	Let terminals
	\[{{\cal C}_0} = \left\{ {{\bf{q}}_0|{\bf{q}}_0 = {{\bf{f}}_s}\left( {\theta ,\rho } \right) = {\cal T}\left( {{s_0},\theta ,\rho } \right),\theta  \in {{\cal D}_\theta },\rho  \in {{\cal D}_\rho }} \right\},\] \[{{\cal C}_1} = \left\{ {{\bf{q}}_m|{\bf{q}}_m = {{\bf{f}}_f}\left( {\theta ,\rho } \right) = {\cal T}\left( {{s_f},\theta ,\rho } \right),\theta  \in {{\cal D}_\theta },\rho  \in {{\cal D}_\rho }} \right\}.\] 
	Thus, denote ${\bf f} = {{\bf f}_f} \circ {\bf f}_s^{ - 1}$ which is a diffeomorphism, and then, the set of order pairs ${\mathcal P}$ is generated by ${\bf f}$. Let the trajectory be 
	\[{\bf{h}}\left( {\left( {{\cal T}\left( {{s_0},\theta ,\rho } \right),{\cal T}\left( {{s_f},\theta ,\rho } \right)} \right),\frac{{s - {s_0}}}{{s_f - {s_0}}}} \right) = {\cal T}\left( {s,\theta ,\rho } \right)\] 
	which is a smooth map. Therefore, this virtual tube with a generator curve is a special case of the virtual tube in \emph{Definition \ref{def:virtual-tube}}. 
	Here is a simple example. As shown in Fig. \ref{fig:virtual_tube_example}, the generator curve ${\bm\gamma} \left( s \right) = \left[s\quad0\quad 0\right]^{\rm{T}}$, radius $\lambda \left( {s,\theta } \right) = 1$, normal unit vector ${\bf{n}}\left( s \right) = \left[0\quad 1\quad 0\right]^{\rm{T}}$, binormal unit vector ${\bf{b}}\left(s\right) = \left[0\quad0\quad1\right]^{\rm{T}}$, $s \in \left[ {0,1} \right]$, $\theta  \in \left[ {0,2\pi } \right)$, $\rho  \in \left[ {0,1} \right]$. 
	Obviously, $\mathcal{C}_0$ and $\mathcal{C}_1$ are disjoint bounded convex subsets, and are called terminals. For any ${\bf{q}}_0 = \left[0\quad{\rho \cos \theta }\quad{\rho \sin \theta }\right]^{\rm{T}} \in {\cal C}_0$, the corresponding point in terminal $\mathcal{C}_1$ is ${\bf{q}}_m = {\bf{f}}\left( {\bf{q}}_0 \right) = \left[1\quad{\rho \cos \theta }\quad{\rho \sin \theta }\right]^{\rm{T}} \in {\cal C}_1$, the set of order pairs $\mathcal{P}$ is constructed. And, ${\bf f}$ is a diffeomorphism, ${\bf h}$ is a smooth map. Thus, this tube with a generator curve could be defined by \emph{Definition \ref{def:virtual-tube}}. 
\end{remark}
}
\begin{figure}
	\centering
	\includegraphics{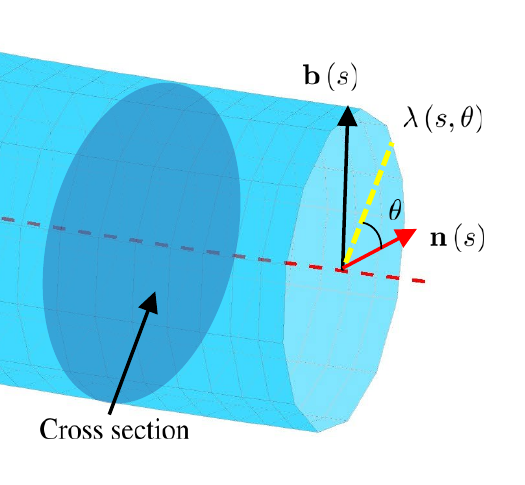}
	\caption{Virtual tube with a generator curve in $\mathbb{R}^3$, where the
		red dotted line is the generator curve; the dark blue plane is a cross-section and
		the light blue surface is the virtual tube surface.}
	\label{fig:virtual_tube_example}
\end{figure}

Some definitions which will be used are proposed in the following. 
\begin{Def}
	A virtual tube $\mathcal{T}\left( {\mathcal{C}_0,\mathcal{C}_1,{\bf{f}},{\bf{h}}} \right)$ is \emph{linear} if it satisfies following properties:
	\begin{itemize}
		\item The diffeomorphism ${\bf f}$ is a \emph{linear} map for any ${\bf q}_{0,k} \in {\mathcal C}_0$ and $\theta_k \in \mathbb{R}$, namely \[{\bf f}\left( {\sum\limits_{k = 1}^q {{\theta _k}{{\bf q}_{0,k}}} } \right) = \sum\limits_{k = 1}^q {{\theta _k}{\bf f}\left( {{{\bf q}_{0,k}}} \right)}.\]
		\item The smooth map ${{\bf h}}$ is a \emph{linear} map for any $\left( {{\bf{q}}_{0,k},{\bf{q}}_{m,k}} \right) \in \mathcal{P}$ and $\eta_k \in \mathbb{R}$, namely
		\[{{\bf{h}}\left( {\sum\limits_{k = 1}^q {{\eta _k}\left( {{{\bf{q}}_{0,k}},{{\bf{q}}_{m,k}}} \right)} ,t} \right)}\]
		\[{ = \sum\limits_{k = 1}^q {{\eta _k}{\bf{h}}\left( {\left( {{{\bf{q}}_{0,k}},{{\bf{q}}_{m,k}}} \right),t} \right).} }\]
	\end{itemize}
	\label{def:linear-tube}
\end{Def}

\begin{Def}
	A terminal $\mathcal{C}_0$ (or $\mathcal{C}_1$) is a \emph{convex hull} of the finite set $\{ {\bf q}_{0,k}\}$ if it can be expressed as
	\begin{equation}
	{\mathcal{C}_0} = \left\{ { {\bf q}_{0} |{\bf q}_{0} = \sum\limits_{k = 1}^q {{\theta _k}{{\bf q}_{0,k}}} ,\sum\limits_{k = 1}^q {{\theta _k}}  = 1,{\theta _k} \ge 0} \right\}.
	\label{equ:convex-hull}
	\end{equation}
	\label{def:convex-hull}
\end{Def}
According to \emph{Definition \ref{def:optimal-trajectory}}, a trajectory ${\bf h}^*$ for an order pair $\left({\bf q}_0,{\bf q}_m\right)$ is optimal with respect to a cost function defined as  $g\left({\bf h}\left(\left({\bf q}_0,{\bf q}_m\right),t\right)\right)$.
Consequently, the definition of an optimal virtual tube is proposed in the following.
\begin{Def}
	A virtual tube $\mathcal T$ is \emph{optimal} with respect to a cost  $g$ if every trajectory in the tube is \emph{optimal} with respect to a cost $g$, namely
	${\cal T} = \left( {\mathcal{C}_0,\mathcal{C}_1,{\bf f},{{\bf h}^*}} \right).$
\end{Def}
As shown in Fig \ref{fig:tubedef}, if all trajectories in the tube are optimal, the tube is an optimal virtual tube.

\section{A Type of Optimal Virtual Tubes: Linear Virtual Tube with Convex Hull Terminals}
There are infinite trajectories in the virtual tube, so it is impossible to solve optimization problems for every optimal trajectory respectively. The first natural thought is to perform interpolation to reduce calculation burden significantly. There are two lemmas about conditions for optimality with proofs shown in \emph{Appendix A}, which are useful for constructing a type of optimal virtual tubes.
\begin{lemma}
	\textbf{Constrained with hyperplanes}: Suppose that a convex optimization problem has the following type
	\begin{equation}
	\begin{array}{ll}
	{\rm{min }}\quad& f_0\left( {\bf{x}} \right)\\
	{\rm{s. t.}} & {\bf{Ax}} = {\bf{b}}
	\end{array}
	\label{equ:optim_problem}
	\end{equation}
	where $f_0\left( {\bf{x}} \right)$ is a convex function with respect to ${\bf{x}} = {\left[ {\begin{array}{*{20}{c}}{{x_0}}&{{x_1}}&{...}&{{x_n}}
			\end{array}} \right]^{\rm{T}}}$, ${\bf{A}} \in \mathbb{R} {^{p\times (n+1)}}$$(p<n+1)$ , ${\bf{b}} \in \mathbb{R} {^{p}}$ and $\nabla f_0\left( {\bf{x}} \right)$ is linear with respect to ${\bf{x}}$. The optimal solution ${\bf{x}}_k$ is obtained when ${\bf{b}} = {{\bf{b}}}_k,k=1,2,...,q$. Denote ${\mathcal{B}} = \{{\bf{b}}_k\}$, ${\mathcal{X}}=\{{\bf{x}}_k\}$,  ${\bf{b}}\left( {\bm{\theta}} \right) =  \sum\nolimits_{k = 1}^q {{\theta _k}{{\bf{b}}_k}} \in {\bf{conv}}\,{\mathcal{B}}$. Then ${\bf{x}}\left( {\bm{\theta}} \right) =\sum\nolimits_{k = 1}^q {{\theta _k}{{\bf{x}}_k}} \in {\bf{conv}}\,{\mathcal{X}}$ is feasible and optimal.
\end{lemma}
\begin{lemma}
	\textbf{Standard form}: Suppose that a convex optimization problem has the following type
	\begin{equation}
	\begin{array}{ll}
	{\rm{min }}\quad &f_0\left( {\bf{x}} \right)\\
	{\rm{s. t. }}\ &{\bf{Ax}} = {\bf{b}}\\
	&{f_i}\left( {\bf{x}} \right)  \le 0,i = 1,...n_c
	\end{array}
	\label{equ:lemma_standrad}
	\end{equation}
	where $f_i\left( {\bf{x}} \right)\left(i=0,1,...,n_c\right)$ are convex functions with respect to ${\bf{x}} = {\left[ {\begin{array}{*{20}{c}}{{x_0}}&{{x_1}}&{...}&{{x_n}}
			\end{array}} \right]^{\rm{T}}}$, ${\bf{A}} \in \mathbb{R} {^{p\times (n+1)}}(p<n+1)$, ${\bf{b}} \in \mathbb{R} {^{p}}$ and $\nabla f_0\left( {\bf{x}} \right)$ is linear with respect to ${\bf{x}}$. The optimal solution ${\bf{x}}_k$ is obtained when ${\bf{b}} = {{\bf{b}}}_k$, $k=1,2,...,q$. Denote ${\mathcal{B}} = \{{\bf{b}}_k\}$, ${\mathcal{X}}=\{{\bf{x}}_k\}$, ${\bf{b}}\left( {\bm{\theta}} \right) =  \sum\nolimits_{k = 1}^q {{\theta _k}{{\bf{b}}_k}} \in {\bf{conv}}\,{\mathcal{B}}$. Then ${\bf{x}}\left( {\bm{\theta}} \right) =\sum\nolimits_{k = 1}^q {{\theta _k}{{\bf{x}}_k}} \in {\bf{conv}}\,{\mathcal{X}}$ is feasible and optimal.
	\label{lemma:standard_form}
\end{lemma}
\begin{Def}
	A virtual tube $\mathcal{T}\left(\mathcal{C}_0,\mathcal{C}_1,{\bf f}, {\bf h}\right)$ is called \emph{a linear virtual tube with convex hull terminals} if it satisfies:
	\begin{itemize}
		\item The terminal $\mathcal{C}_0$ is the convex hull of the finite set $\{{\bf q}_{0,k}\}$, and the terminal $\mathcal{C}_1$ is the convex hull of the finite set $\{{\bf q}_{m,k}\}$, $k=1,2,...,q$. The set $\{{\bf q}_{0,k}\}$ and the set $\{{\bf q}_{m,k}\}$ have the same number of elements.
		\item For any ${{{\bf{q}}_{0,k}}}$ in the set $\{{\bf q}_{0,k}\}$, ${{\bf{q}}_{m,k}} = {\bf{f}}\left( {{{\bf{q}}_{0,k}}} \right)$.
		\item The virtual tube $\mathcal{T}\left(\mathcal{C}_0,\mathcal{C}_1,{\bf f}, {\bf h}\right)$ is \emph{linear}.
	\end{itemize}
\end{Def}
A linear virtual tube with convex hull terminals is a special case of virtual tubes, and the convex hull of the finite set is a polyhedron (\cite{Boyd_convex_2004}), as shown in Fig. \ref{fig:cross-section}. An example of a linear virtual tube with convex hull terminals is shown in Fig. \ref{fig:tubedef}. 
\begin{figure}
	\centering
	\includegraphics{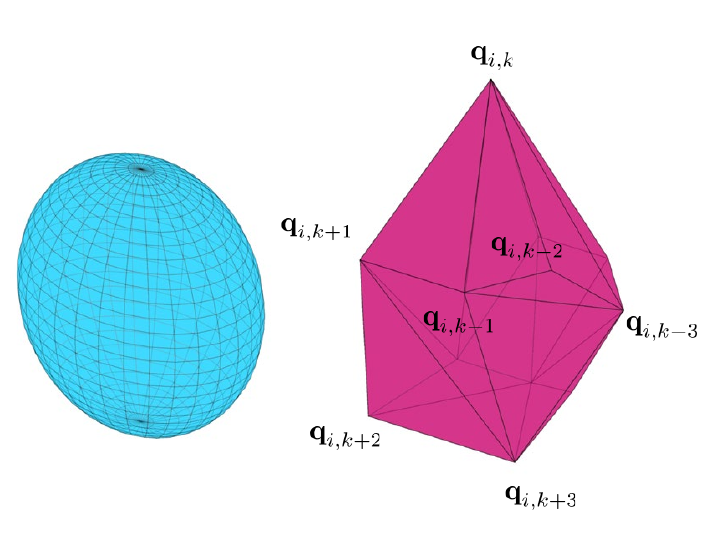}
	\caption{The blue sphere is a terminal which is a convex set. The purple polyhedron is a convex hull terminal. And they could also indicate the cross-section ${\mathcal{C}}_t$ at $t$.}
	\label{fig:cross-section}
\end{figure}
\begin{theorem}
	A linear virtual tube with convex hull terminals $\mathcal{T}\left(\mathcal{C}_0,\mathcal{C}_1,{\bf f}, {\bf h}\right)$ is optimal with respect to cost $f_0$ if (i) its trajectory for each order pair ${\left({\bf q}_0,{\bf q}_m\right)}$ can be linearly parameterized as ${\bf h}\left( {\left( {{\bf{q}}_0,{\bf{q}}_m} \right),t} \right) = {\bf{C}}\left( t \right){\bf{x}}$, where ${\bf C}\left( t \right)$ is a matrix only respect to $t$, and ${\bf x}$ is a parameter vector, (ii) the trajectories ${\bf h}^*\left( {\left( {{\bf{q}}_{0,k},{\bf{q}}_{m,k}} \right),t} \right)$ for order pairs $\left( {{{\bf{q}}_{0,k}},{{\bf{q}}_{m,k}}} \right)\left(k=1,2,...,q\right)$ are optimal with respect to cost $f_0$ and satisfy conditions in (\ref{equ:lemma_standrad}) where ${\bf b} = {\bf b}_k, {\bf x}={\bf x}_k$ for each trajectory ${\bf h}^*\left( {\left( {{\bf{q}}_{0,k},{\bf{q}}_{m,k}} \right),t} \right)$, (iii) the linear equality constraint is ${\bf{Ax}} = \sum\nolimits_{k = 1}^q {{\theta _k}{{\bf{b}}_k}} $ for each order pair $\left( {{{\bf{q}}_0},{{\bf{q}}_m}} \right)$.
	\label{the:convex_hull}
\end{theorem}
\begin{proof}
	We show that, in \emph{Step 1}, any trajectory ${\bf h}\left( {\left( {{\bf{q}}_0,{\bf{q}}_m} \right),t} \right)$ in the linear virtual tube with convex hull terminals can be expressed as a convex combination of trajectories ${\bf h}\left( {\left( {{\bf{q}}_{0,k},{\bf{q}}_{m,k}} \right),t} \right)$. Then, in \emph{Step 2}, based on \emph{Lemma \ref{lemma:standard_form}}, we will show that any trajectory ${\bf h}\left( {\left( {{\bf{q}}_0,{\bf{q}}_m} \right),t} \right)$ is optimal if the trajectories ${\bf h}^*\left( {\left( {{\bf{q}}_{0,k},{\bf{q}}_{m,k}} \right),t} \right)$ for order pairs $\left( {{{\bf{q}}_{0,k}},{{\bf{q}}_{m,k}}} \right)\left(k=1,2,...,q\right)$ are optimal. Therefore, this virtual tube is optimal.
	  
	
	\textbf{Step 1}. Suppose that the terminal ${\mathcal C}_0$ is the convex hull of the finite set $\left\{ {{{\bf{q}}_{0,k}}} \right\}$, such as 
	\[{\mathcal{C}_0} = \left\{ { {\bf q}_0 |{\bf q}_0 = \sum\limits_{k = 1}^q {{\theta _k}{{\bf q}_{0,k}}} ,\sum\limits_{k = 1}^q {{\theta _k}}  = 1,{\theta _k} \ge 0} \right\}, \]
	where $q$ is the number of the elements in set $\left\{ {{{\bf{q}}_{0,k}}} \right\}$. Then $q$ order pairs $\{\left( {{{\bf{q}}_{0,k}},{{\bf{q}}_{m,k}}} \right)\}$ are constructed by 
	\[{{\bf{q}}_{m,k}} = {\bf{f}}\left( {{{\bf{q}}_{0,k}}} \right).\]
	Thus, the set of order pairs $\mathcal{P}$ is constructed by 
	\[\begin{array}{lll}
	{\cal P} &=& \left\{ {\left( {{{\bf{q}}_0},{\bf{f}}\left( {{{\bf{q}}_0}} \right)} \right)} \right\}\\
	 &=& \left\{ {\left( {\sum\limits_{k = 1}^q {{\theta _k}{{\bf{q}}_{0,k}}} ,{\bf{f}}\left( {\sum\limits_{k = 1}^q {{\theta _k}{{\bf{q}}_{0,k}}} } \right)} \right)} \right\}\\
	&=& \left\{ {\left( {\sum\limits_{k = 1}^q {{\theta _k}{{\bf{q}}_{0,k}}} ,\sum\limits_{k = 1}^q {{\theta _k}{\bf{f}}\left( {{{\bf{q}}_{0,k}}} \right)} } \right)} \right\}\\
	&=&\left\{ {\left( {\sum\limits_{k = 1}^q {{\theta _k}{{\bf{q}}_{0,k}}} ,\sum\limits_{k = 1}^q {{\theta _k}{{\bf{q}}_{m,k}}} } \right)} \right\}\\
	&=&\left\{ {\sum\limits_{k = 1}^q {{\theta _k}\left( {{{\bf{q}}_{0,k}},{{\bf{q}}_{m,k}}} \right)} } \right\}.
	\end{array}\]
	Any trajectory for the order pair $\left({\bf q}_0,{\bf q}_m\right)$ in $\mathcal{P}$ is expressed as 
	\begin{equation}
	\begin{array}{*{20}{l}}
	{{\bf{h}}\left( {\left( {{{\bf{q}}_0},{{\bf{q}}_m}} \right),t} \right) = {\bf{h}}\left( {\sum\limits_{k = 1}^q {{\theta _k}\left( {{{\bf{q}}_{0,k}},{{\bf{q}}_{m,k}}} \right)} ,t} \right)}\\
	{ = \sum\limits_{k = 1}^q {{\theta _k}{\bf{h}}\left( {\left( {{{\bf{q}}_{0,k}},{{\bf{q}}_{m,k}}} \right),t} \right)} }
	\end{array}
	\end{equation}
	
	\textbf{Step 2}. We will show that any trajectory in the linear virtual tube with convex hull terminals is optimal, namely, this virtual tube is optimal.
	
	Through solving optimization problem in (\ref{equ:lemma_standrad}) where ${\bf{b}} = {{\bf{b}}_k}$, the $q$ optimal trajectories for order pairs $\{\left( {{{\bf{q}}_{0,k}},{{\bf{q}}_{m,k}}} \right)\}$ are generated, which are linearly parameterized as
	\begin{equation}
		{\bf h}^*\left( {\left( {{{\bf{q}}_{0,k}},{{\bf{q}}_{m,k}}} \right),t} \right) = {\bf{C}}\left( t \right){{\bf{x}}_k},k = 1,...,q,
		\label{equ:b-spline-matrix-representation}
	\end{equation}
	where ${\bf{C}}\left( t \right)$ is a matrix with $t$, ${\bf{x}}_k$ is the optimal solution. 
	
	Because the linear constraint for each order pair $\left({\bf q}_0, {\bf q}_m\right)$ is ${\bf{Ax}} = \sum\nolimits_{k = 1}^q {{\theta _k}{{\bf{b}}_k}} $, 
	according to \emph{Lemma \ref{lemma:standard_form}}, ${\bf x}=\sum\nolimits_{k = 1}^q {{\theta _k}{{\bf{x}}_k}} $ is the optimal solution. 	
	Therefore, any trajectory for order pair $\left( {\bf q}_0, {\bf q}_m \right) \in {\mathcal P}$ is optimal which is expressed as
	\[\begin{array}{*{20}{l}}
	{{{\bf{h}}^*}\left( {\left( {{{\bf{q}}_0},{{\bf{q}}_m}} \right),t} \right) = \sum\limits_{k = 1}^q {{\theta _k}{{\bf{h}}^*}\left( {\left( {{{\bf{q}}_{0,k}},{{\bf{q}}_{m,k}}} \right),t} \right)} }\\
	{ = \sum\limits_{k = 1}^q {{\theta _k}{\bf{C}}\left( t \right){{\bf{x}}_k}}  = {\bf{C}}\left( t \right)\sum\limits_{k = 1}^q {{\theta _k}{{\bf{x}}_k}}  = {\bf{C}}\left( t \right){\bf{x}}.}
	\end{array}\]
	Thus, this linear virtual tube with convex hull terminals $\left( {\mathcal{C}_0,\mathcal{C}_1,{\bf f},{{\bf h}^*}} \right)$ is optimal. $\hfill\blacksquare$
\end{proof}
An example of the optimal linear virtual tube with convex hull terminals is given as the following.

\begin{exmp}
	In two-dimension Euclidean space, suppose that there exists a linear virtual tube with convex hull terminals $\mathcal{T}\left(\mathcal{C}_0,\mathcal{C}_1,{\bf f},{\bf h}\right)$. The terminals $\mathcal{C}_0$, $\mathcal{C}_1$ are expressed as the convex hull of the the set $\left\{ {{{\bf{q}}_{0,1}},{{\bf{q}}_{0,2}}} \right\}$ and $\left\{ {{{\bf{q}}_{m,1}},{{\bf{q}}_{m,2}}} \right\}$, namely
	${\mathcal{C}_0}={\bf q}_0\left(\theta\right) = \left( {1 - \theta} \right){{\bf{q}}_{0,1}} + \theta{{\bf{q}}_{0,2}}$, ${\mathcal{C}_1}={\bf q}_m\left(\theta\right) = \left( {1 - \theta} \right){{\bf{q}}_{m,1}} + \theta{{\bf{q}}_{m,2}}$, $\theta \in \left[ {0,1} \right]$. The smooth map ${\bf f}$ is defined as ${{\bf{q}}_{m,k}} = {\bf{f}}\left( {{{\bf{q}}_{0,k}}} \right),k = 1,2$. Thus, a set of order pairs is expressed as ${\mathcal{P}} = \left\{ {\left( {{\bf q}_0\left( \theta \right),{{\bf q}_m}\left( \theta \right)} \right)} \right\}$. 
	Let ${{\bf{h}}_0^*}\left( t \right),{{\bf{h}}_1^*}\left( t \right)$ denote the optimal trajectories for order pairs $\left( {{{\bf{q}}_{0,1}},{{\bf{q}}_{m,1}}} \right)$ and $\left( {{{\bf{q}}_{0,2}},{{\bf{q}}_{m,2}}} \right)$ respectively which are generated by solving optimization problems satisfying \emph{Lemma} \ref{lemma:standard_form}. And the linear equality constraint for ${\left( {{{\bf q}_0}\left( \theta \right),{{\bf q}_m}\left(\theta\right)} \right)}$ satisfies condition (iii) in \emph{Theorem} \ref{the:convex_hull}. Thus, the optimal trajectory of any order pair ${\left( {{{\bf q}_0}\left( \theta \right),{{\bf q}_m}\left(\theta\right)} \right)}\in \mathcal{P}$ is expressed as ${\bf h}_\theta^*\left( t \right) = \left( {1 - \theta} \right){{\bf{h}}_0^*}\left( t \right) + \theta{{\bf{h}}_1^*}\left( t \right)$. Since every trajectory in the virtual tube is optimal, $\mathcal{T}\left(\mathcal{C}_0,\mathcal{C}_1,{\bf f},{\bf h}^*\right)$ is the optimal virtual tube, as shown in Fig. \ref{fig:optimal-tube}.
\end{exmp}
\begin{figure}
	\centering
	\includegraphics{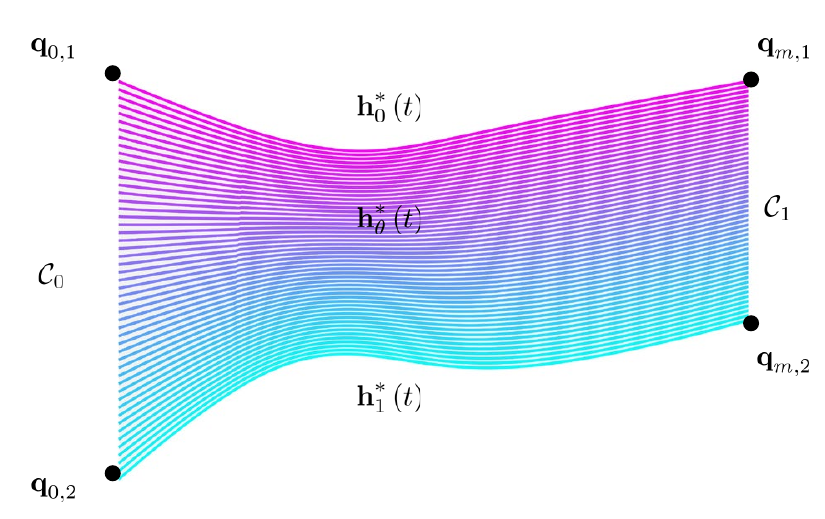}
	\caption{An optimal virtual tube example. Different colors represent different optimal trajectories. Curves ${\bf h}_0^*\left(t\right)$ and ${\bf h}_1^* \left( t \right)$ are optimal trajectories generated by solving optimization problems. And ${\bf h}_\theta^*\left(  t \right)$ are optimal trajectories generated by interpolation. }
	\label{fig:optimal-tube}
\end{figure}

\section{A Planning Method to Obtain an Optimal Virtual Tube}
Considering a robot swarm composed of $M$ agents labeled by $j \in \mathcal{M} = \left\{ {1,...,M} \right\}$, a start area which is the convex hull of set $\{ {\bf q}_{0,k}\}$ is labeled by $\mathcal{C}_0$, and a goal area which is the convex hull of set $\{ {\bf q}_{m,k}\}$ is labeled by $\mathcal{C}_1$. Suppose that the start position ${\bf q}_{0}\left({\bm \theta}_j\right) \in \mathcal{C}_0$ of robot $j$ is given. An obstacle-dense environment is between the start area and the goal area. The objective is to construct an optimal virtual tube and make robots pass through the obstacle-dense environment in the optimal virtual tube.
\subsection{Outline}
For an optimal virtual tube ${\cal T}\left( {{{\cal C}_0},{{\cal C}_1},{\bf{f}},{{\bf{h}}^*}} \right)$,
the set of order pairs $\mathcal{P}$ is constructed by ${\bf f}$ first, after inputting terminals $\mathcal{C}_0$ and $\mathcal{C}_1$. Then the paths for order pairs ${\left( {{{\bf{q}}_{0,k}},{{\bf{q}}_{m,k}}} \right)}$ are found. Next, the optimal trajectories ${{{\bf{h}}^*}\left( {\left( {{{\bf{q}}_{0,k}},{{\bf{q}}_{m,k}}} \right),t} \right)}$ for order pairs ${\left( {{{\bf{q}}_{0,k}},{{\bf{q}}_{m,k}}} \right)}$ are generated by solving optimization problems. Finally,
an optimal linear virtual tube with convex hull terminals ${\cal T}\left( {{{\cal C}_0},{{\cal C}_1},{\bf{f}},{{\bf{h}}^*}} \right)$ is constructed, with an infinite number of optimal trajectories. The whole process of optimal virtual tube planning is shown in Fig. \ref{fig:tubeplanningprocess}.
\begin{figure*}
	\centering
	\includegraphics{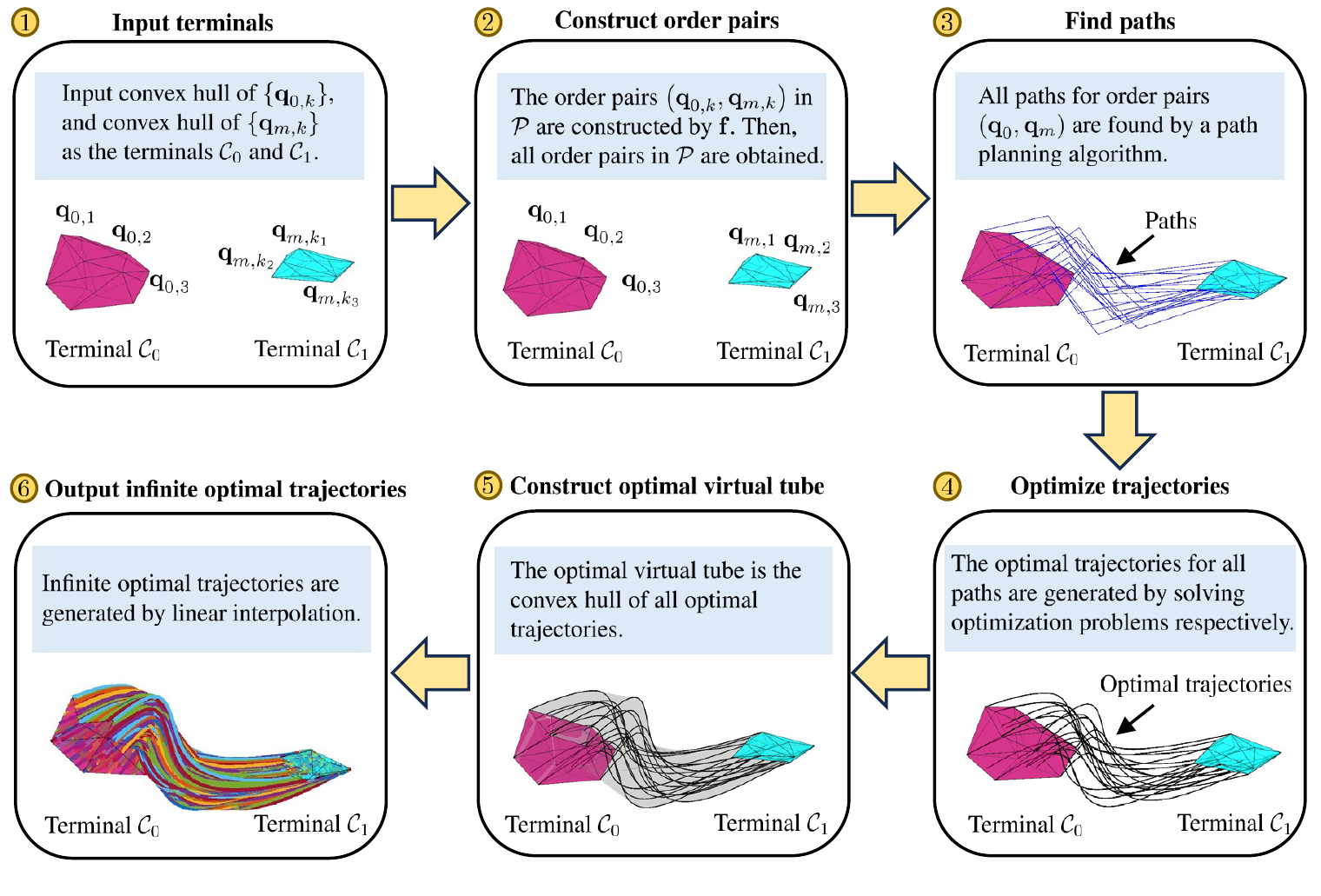}
	\caption{{Optimal virtual tube planning process.}}
	\label{fig:tubeplanningprocess}
\end{figure*}
\subsection{Constructing order pairs}
{
	Suppose that the start area and goal area called terminal $\mathcal{C}_0$ and $\mathcal{C}_1$, are the convex hull of $\left\{ {{{\bf{q}}_{0,k}}} \right\}$ and $\left\{ {{{\bf{q}}_{m,k}}} \right\}$ respectively with the same number of extreme points (\cite{adasch2006topological}), as shown in Step \textcircled{1} of Fig. \ref{fig:tubeplanningprocess}. 
	{Since the order pairs in $\mathcal{P}$ are constructed by a map ${\bf f}$,}
	a selection strategy of order pairs $\left({\bf q}_{0,k}, {\bf q}_{m,k}\right)$, namely a suitable map ${\bf f}$, is simply described in the following. Suppose the number of extreme points of terminals is $q$. First, select the extreme points (or so called vertices) of the terminal $\mathcal{C}_0$ as the set of $\{{\bf q}_{0,k}\}\left(k=1,2,...,q\right)$. Then, the order pairs $\left({\bf q}_{0,k}, {\bf q}_{m,k}\right)$ are constructed from a map ${\bf f}$, where ${\bf q}_{m,k} = {\bf f}\left({\bf q}_{0,k}\right)$ are also the vertices of terminal $\mathcal{C}_1$ and the sorting principle of ${\bf q}_{m,k}\left(k=1,2,...,q\right)$ is specified by {a classic assignment problem (AP) (\cite{pentico2007assignment}) in which the objective function minimizes both the total distance and the variance of distance between ${\bf q}_{0,k}$ and ${\bf q}_{m,k}$}, as shown in Step \textcircled{2} of Fig. \ref{fig:tubeplanningprocess}. Then, the set of order pairs $\mathcal{P}$ is constructed from the map ${\bf f}$. 
}

Any start position ${\bf q}_0\left({\bm \theta}_j\right)$ in terminal $\mathcal{C}_0$ is expressed as
\begin{equation}
{\bf q}_0\left({\bm \theta}_j\right) = \sum\limits_{k = 1}^q {{\theta _{kj}}{{\bf{q}}_{0,k}}} ,\sum\limits_{k = 1}^q {{\theta _{kj}}}  = 1,{\theta _{kj}} \ge 0,j \in \mathbb{Z}^{+},
\label{equ:start_position}
\end{equation} 
where $q$ is a finite number of elements in $\left\{ {{{\bf{q}}_{0,k}}} \right\}$. Compared with the infinite number of start positions in terminal $\mathcal{C}_0$, the number of elements in $\left\{ {{{\bf{q}}_{0,k}}} \right\}$ is small enough. Let ${\bf f}$ be the linear map between $\mathcal{C}_0$ and $\mathcal{C}_1$ so that for any goal position ${\bf q}_m\left({\bm \theta}_j\right) \in \mathcal{C}_1$,
\[\begin{array}{l}
{\bf q}_m\left({\bm \theta}_j\right) = {\bf f}\left( {\bf q}_0\left({\bm \theta}_j\right) \right) = {\bf f}\left( {\sum\limits_{k = 1}^q {{\theta _{kj}}{{\bf{q}}_{0,k}}} } \right) = \sum\limits_{k = 1}^q {{\theta _{kj}}{\bf f}\left( {{{\bf{q}}_{0,k}}} \right)}\\
  = \sum\limits_{k = 1}^q {{\theta _{kj}}{{\bf{q}}_{m,k}}}.
\end{array}\]  
The set of order pairs $\mathcal{P}$, thus, has been constructed, such as $\{\left( {{{\bf{q}}_{0}\left({\bm \theta}_j\right)},{{\bf{q}}_{m}\left({\bm \theta}_j\right)}} \right)\}.$
\subsection{Finding paths}
The paths for order pairs $\left( {{{\bf{q}}_{0,k}},{{\bf{q}}_{m,k}}} \right)\left(k=1,...,q\right)$ are needed to be found after constructing order pairs. And a normalized parameterization of the paths is designed.

There are many path-finding methods used for finding waypoints in the obstacle environment such as RRT, RRT* (\cite{karaman2011sampling}), and A*. For each order pair $\left( {{{\bf{q}}_{0,k}},{{\bf{q}}_{m,k}}} \right)\left(k=1,...,q\right)$, we select RRT* to find $m+1$ waypoints $\left\{ {{{\bf{q}}_{i,k}}} \right\}\left(i=0,...,m \right)$ between ${\bf q}_{0,k}$ and ${\bf q}_{m,k} $. 

To parameterize the waypoints $\left\{ {{{\bf{q}}_{i,k}}} \right\}$, parameters $ u_{i,k}$ are assigned to corresponding waypoint $ {{{\bf{q}}_{i,k}}}$. Thus, a partition of $u_k$ is first designed, called knots $\{ u_{i,k} \}$. Then public knots $\{u_i\}$ and public normalized knots $\{{t_{k}}\}$ are generated by domain transformation.
The chord length parameterization, a suitable method in most engineering applications (\cite{piegl2000surface}), is used to generate knots $\{ u_{i,k} \}$ expressed as
\begin{equation}
{u_{i,k}} = \left\{ {\begin{array}{*{20}{c}}
	{0,}&{i = 0,}\\
	{{u_{i - 1,k}} + \left\| {{{\bf{q}}_{i,k}} - {{\bf{q}}_{i - 1,k}}} \right\|,}&{i = 1,2,...,m.}
	\end{array}} \right.
\end{equation}
It should be noted that the knots $\{ u_{i,k} \}$ for each pair  $\left( {{{\bf{q}}_{0,k}},{{\bf{q}}_{m,k}}} \right)$ are generally not the same, because of variation in distances $ \left\| {{{\bf{q}}_{i,k}} - {{\bf{q}}_{i - 1,k}}} \right\|,k=1,...,q.$ Motivated by the parameterization of tensor product surface (\cite{STROTHOTTE2002203}), all parameterizations of paths are expected to be the same for convenience. The public knots $\left\{ {{u_i}} \right\}$ are the arithmetic mean of all knots, which are expressed as
\begin{equation}
{u_i} = \frac{{\sum\limits_{k = 1}^q {{u_{ki}}} }}{q}.
\end{equation}
To get the normalized parameterization for correspondence with \emph{Definition \ref{def:virtual-tube}}, the normalized knots $\left\{ {{t_i}} \right\}$ are obtained by 
\begin{equation}
{t_i} = \frac{{{u_i}}}{{{u_m}}},i = 0,1,...,m.
\end{equation}

Connecting adjacent waypoints with straight lines, the RRT* path planner finds $q$ collision-free paths for $q$ order pairs $\{{\left( {{{\bf{q}}_{0,k}},{{\bf{q}}_{m,k}}} \right)}\}$ respectively, as shown in Step \textcircled{3} of Fig. \ref{fig:tubeplanningprocess}.
\begin{remark}
	The numbers of the waypoints $\left\{ {{{\bf{q}}_{i,k}}} \right\}$ of different order pairs are assumed as the same. Otherwise, different numbers of the waypoints $\left\{ {{{\bf{q}}_{i,k}}} \right\}$ lead to different lengths of parameter vectors ${\bf x}$ in \emph{Theorem 1}, that would not satisfy \emph{Lemmas 1,2}. {Meanwhile, through shrinking the sample space around the first found path to eliminate any potential detours or alternate routes, all paths could be homotopic so that there is no obstacle between paths.}
\end{remark}
{
\begin{remark}
	The scale of the swarm would be limited in an extremely cluttered environment. The swarm is not allowed to split apart to move around obstacles because of the homotopic paths. Thus, in an extremely cluttered environment, all paths need to pass through a narrow gap whose size is much smaller than that of the swarm, which would cause blockage of the movement of the swarm. To overcome this weak point, large swarms could be divided into several subgroups satisfying the minimum size of gaps to plan virtual tubes.
\end{remark}
}
\subsection{Optimizing trajectories}
The original paths in the above subsection have some sharp corners, as shown in Step \textcircled{3} of Fig. \ref{fig:tubeplanningprocess}, which are not feasible considering the dynamics of the robots. It is necessary to improve the smoothness of the original paths, as shown in Step \textcircled{4} of Fig. \ref{fig:tubeplanningprocess}.
\subsubsection{Trajectory representation}
There are many ways to represent trajectories linearly, such as B-spline and polynomials. 
It is convenient to use $n$-th order piecewise polynomial to represent the trajectory ${\bf h}$, such as
\begin{equation}
{\bf{h }}\left( t \right) = \left\{ {\begin{array}{*{20}{c}}
	{\sum\limits_{i = 0}^n {{{\bf{a}}_{i,1}}{t^i}} }&{{t_0} < t < {t_1}}\\
	{\sum\limits_{i = 0}^n {{{\bf{a}}_{i,2}}{t^i}} }&{{t_1} < t < {t_2}}\\
	\vdots & \vdots \\
	{\sum\limits_{i = 0}^n {{{\bf{a}}_{i,m}}{t^i}} }&{{t_{m - 1}} < t < {t_m}}
	\end{array}} \right.
\label{poly_nomial}
\end{equation}
where ${{\bf{a}}_{i,j}} = {\left[ {\begin{array}{*{20}{c}}
		{{a_{1,i,j}}}&{{a_{2,i,j}}}& \ldots &{{a_{d,i,j}}}
		\end{array}} \right]^{\rm{T}}} \in \mathbb{R}^d, j=1,2,3,...,m$. For any $t\in \left[t_{k-1},t_k\right]$, the trajectory ${\bf h}\left(t\right)$ could be represented by matrix
\begin{equation}
{\bf{h}}_j\left( t \right) = {\bf{C}}\left( t \right){\bf{x}}_j
\label{equ:matrix_path}
\end{equation}
where ${\bf{C}}\left( t \right) = \left[ {\begin{array}{*{20}{c}}
	{{{\bf{I}}_d}}&{t{{\bf{I}}_d}}&{{t^2}{{\bf{I}}_d}}& \ldots &{{t^n}{{\bf{I}}_d}}
	\end{array}} \right]$,\[{\bf{x}}_j = \left[ {\begin{array}{*{20}{c}}
	{{{\bf{a}}_{0,j}}}\\
	{{{\bf{a}}_{1,j}}}\\
	{{{\bf{a}}_{2,j}}}\\
	\vdots \\
	{{{\bf{a}}_{n,j}}}
	\end{array}} \right].\] 
The representation of trajectory (\ref{equ:matrix_path}) is convenient to be used in quadratic programming (QP) in the following.
\subsubsection{Constraints}
The constraints, including the terminal conditions and intermediate conditions, are constructed as linear equations for smoothness. In terms of obstacle avoidance, adding corridor constraints could avoid collisions.

The terminal conditions are expressed as
\begin{equation}
\begin{array}{*{20}{l}}
{{{\bf{h}}_1}\left( {{t_0}} \right) = {\bf{C}}\left( {{t_0}} \right){{\bf{x}}_1} = {{\bf{q}}_{0,k}},}\\
{{\bf{h}}_m}\left( {{t_m}} \right) = {\bf{C}}\left( {{t_m}} \right){{\bf{x}}_m} = {{\bf{q}}_{m,k}},\\
{\left. {\frac{{{{\rm{d}}^p}{{\bf{h}}_1}\left( t \right)}}{{{\rm{d}}{t^p}}}} \right|_{t = {t_0}}} = {\bf{q}}_{0,k}^{\left(p\right)},\\
{\left. {\frac{{{{\rm{d}}^p}{{\bf{h}}_m}\left( t \right)}}{{{\rm{d}}{t^p}}}} \right|_{t = {t_m}}} = {\bf{q}}_{m,k}^{\left(p\right)},
\end{array}
\label{equ:cons_ter}
\end{equation}
where ${\bf q}^{\left(p\right)}_{0,k}$ and ${\bf q}^{\left(p\right)}_{m,k}$ are the $p$-order conditions in the start and goal points, terminals of the trajectory are constrained in start and goal points. And the intermediate conditions are expressed as
\begin{equation}
{\left. {\frac{{{{\rm{d}}^p}{{\bf{h}}_i}\left( t \right)}}{{{\rm{d}}{t^p}}}} \right|_{t = {t_i}}} = {\left. {\frac{{{{\rm{d}}^p}{{\bf{h}}_{i + 1}}\left( t \right)}}{{{\rm{d}}{t^p}}}} \right|_{t = {t_i}}},
\label{equ:cons_inter}
\end{equation}
where $p=0,1,2,...k_r,i=1,2,...,m-1$,
\[{\left. {\frac{{{{\rm{d}}^p}{{\bf{h}}_i}\left( t \right)}}{{{\rm{d}}{t^p}}}} \right|_{t = {t_i}}} = {\left. {\frac{{{{\rm{d}}^p}{\bf{C}}\left( t \right)}}{{{\rm{d}}{t^p}}}} \right|_{t = {t_i}}}{{\bf{x}}_i},\]
\[{\left. {\frac{{{{\rm{d}}^p}{{\bf{h}}_{i + 1}}\left( t \right)}}{{{\rm{d}}{t^p}}}} \right|_{t = {t_i}}} = {\left. {\frac{{{{\rm{d}}^p}{\bf{C}}\left( t \right)}}{{{\rm{d}}{t^p}}}} \right|_{t = {t_i}}}{{\bf{x}}_{i + 1}}.\]
Specifically, when $p=0$, the terminals of each segment are waypoints:
\begin{equation}
{{\bf{h}}_i}\left( {{t_i}} \right) = {{\bf{h}}_{i + 1}}\left( {{t_i}} \right) = {{\bf{q}}_{i,k}}.
\label{equ:con_ter_first}
\end{equation}
For convenience, we denote \[{{\bf{C}}_t^{\left( p \right)}} = \frac{{{{\rm{d}}^p}{\bf{C}}\left( t \right)}}{{{\rm{d}}{t^p}}}\]
which could transform (\ref{equ:cons_inter}) into
\begin{equation}
\left[ {\begin{array}{*{20}{c}}
	{{{\bf{C}}_{{t_i}}^{\left( p \right)}}}&{ - {{\bf{C}}_{{t_i}}^{\left( p \right)}} }
	\end{array}} \right]\left[ {\begin{array}{*{20}{c}}
	{{{\bf{x}}_i}}\\
	{{{\bf{x}}_{i + 1}}}
	\end{array}} \right] = {\bf{0}}.
\label{equ:cons_inter_mat}
\end{equation}
Based on (\ref{equ:cons_ter}), (\ref{equ:con_ter_first}), and (\ref{equ:cons_inter_mat}), the linear equality constraints are derived
\begin{equation}
{\bf{Ax}} = \left[ {\begin{array}{*{20}{c}}
	{{{\bf{A}}_1}}\\
	{{{\bf{A}}_2}}\\
	{{{\bf{A}}_3}}
	\end{array}} \right]{\bf{x}} = {\left[ {\begin{array}{*{20}{c}}
		{{{\bf{b}}_1}}\\
		{{{\bf{b}}_2}}\\
		{{{\bf{b}}_3}}
		\end{array}} \right]_k} = {{\bf{b}}_k},
\label{equ:equ_cons}
\end{equation}
where 
\[{\bf{x}} = \left[ {\begin{array}{*{20}{c}}
	{{{\bf{x}}_1}}\\
	\vdots \\
	{{{\bf{x}}_{m}}}
	\end{array}} \right],{{\bf{b}}_1} = \left[ {\begin{array}{*{20}{c}}
	{\bf{0}}\\
	\vdots \\
	{\bf{0}}
	\end{array}} \right],{{\bf{b}}_2} = \left[ {\begin{array}{*{20}{c}}
	{{{\bf{q}}_{0,k}}}\\
	{{{\bf{q}}_{1,k}}}\\
	\vdots \\
	{{{\bf{q}}_{m - 1,k}}}\\
	{{{\bf{q}}_{m,k}}}
	\end{array}} \right],\]
\[{{\bf{b}}_3} = {\left[ {\begin{array}{*{20}{c}}
		{{\bf{q}}_{0,k}^{\left( p \right)}}& \cdots &{{\bf{q}}_{0,k}^{\left( 1 \right)}}&{{\bf{q}}_{m,k}^{\left( p \right)}}& \cdots &{{\bf{q}}_{m,k}^{\left( 1 \right)}}
		\end{array}} \right]^{\rm{T}}},\]
\[{\bf{A}}_1 = \left[ {\begin{array}{*{20}{c}}
	{{\bf{C}}_{{t_1}}^{\left( p \right)}}&{ - {\bf{C}}_{{t_1}}^{\left( p \right)}}& \cdots &{\bf{0}}&{\bf{0}}\\
	\vdots & \vdots & \cdots & \vdots & \vdots \\
	{\bf{0}}&{\bf{0}}& \cdots &{{\bf{C}}_{{t_{m-1}}}^{\left( p \right)}}&{ - {\bf{C}}_{{t_{m-1 }}}^{\left( p \right)}}
	\end{array}} \right],\]
\[{{\bf{A}}_2} = \left[ {\begin{array}{*{20}{c}}
	{{{\bf{C}}_{{t_0}}}}&{\bf{0}}& \cdots &{\bf{0}}\\
	{\bf{0}}&{{{\bf{C}}_{{t_1}}}}& \cdots &{\bf{0}}\\
	\vdots & \vdots & \vdots & \vdots \\
	{\bf{0}}&{\bf{0}}& \cdots &{{{\bf{C}}_{{t_{m - 1}}}}}\\
	{\bf{0}}&{\bf{0}}& \cdots &{{{\bf{C}}_{{t_m}}}}
	\end{array}} \right],\]
\[{{\bf{A}}_3} = \left[ {\begin{array}{*{20}{c}}
	{{\bf{C}}_{{t_0}}^{\left( p \right)}}&{\bf{0}}& \cdots &{\bf{0}}&{\bf{0}}\\
	\vdots & \vdots & \ddots & \vdots & \vdots \\
	{{\bf{C}}_{{t_0}}^{\left( 1 \right)}}&{\bf{0}}& \cdots &{\bf{0}}&{\bf{0}}\\
	{\bf{0}}&{\bf{0}}& \cdots &{\bf{0}}&{{\bf{C}}_{{t_m}}^{\left( p \right)}}\\
	\vdots & \vdots & \ddots &{\bf{0}}& \vdots \\
	{\bf{0}}&{\bf{0}}& \cdots &{\bf{0}}&{{\bf{C}}_{{t_m}}^{\left( 1 \right)}}
	\end{array}} \right].\]

To avoid collision between trajectories and obstacles, corridor constraints (\cite{mellinger_minimum_2011}) are added by constraining $n_j$ intermediate points in the segment from ${\bf q}_{i,k}$ to ${\bf q}_{i+1,k}$ $\left(i=0,1,2,...,m-1\right)$. For an intermediate point of trajectory ${\bf{h}}\left( {{s_j}} \right)$ in the segment from ${\bf q}_{i,k}$ to ${\bf q}_{i+1,k}$, the perpendicular distance vector between ${\bf{h}}\left( {{s_j}} \right)$ and the segment is expressed as
\[{{\bf{d}}_i}\left( {{s_j}} \right) = \left( {{\bf{h}}\left( {{s_j}} \right) - {{\bf{q}}_{i,k}}} \right) - \left( {\left( {{\bf{h}}\left( {{s_j}} \right) - {{\bf{q}}_{i,k}}} \right)^{\rm T} {{\bf{t}}_i}} \right){{\bf{t}}_i},\]
where ${\bf t}_i$ is the unit vector along segment from ${\bf q}_{i,k}$ to ${\bf q}_{i+1,k}$.
Thus, the convex constraints are expressed as:
\begin{equation}
{f_i}\left( {\bf{x}} \right) = {\left\| {{{\bf{d}}_i}\left( {{s_j}} \right)} \right\|_\infty } - {\delta _i} \le 0,
\label{equ:con_equ}
\end{equation}
where ${s_j} = {t_i} + \frac{j}{{1 + {n_c}}}\left( {{t_{i + 1}} - {t_i}} \right)$, $j = 1,...,{n_c}$, ${\delta _i}$ is the corridor width, ${{{\bf{d}}_i}\left( {{s_j}} \right)}$ is the perpendicular distance vector.

\subsubsection{Cost function}
The optimization problem to minimize energy cost could be formulated as a QP by regarding ${\bf x}$ in (\ref{equ:equ_cons}) as variable vectors (\cite{mellinger_minimum_2011}). The cost function is expressed as
\begin{equation}
\begin{array}{rcl}
{E_s} &  = &\int\limits_{{t_0}}^{{t_m}} {{{\left\| {\frac{{{{\rm{d}}^{k_r}}{\bf{h}}\left( t \right)}}{{{\rm{d}}{t^{k_r}}}}} \right\|}^2}{\rm{d}}t} \\
&  = &\sum\limits_{i = 1}^m {{\bf{x}}_i^{\rm{T}}\int\limits_{{t_{i - 1}}}^{{t_i}} {{{{\bf{C}}_t^{(k_r)}}^{\rm{T}}}{\bf{C}}_t^{(k_r)}{\rm{d}}t} {{\bf{x}}_i}} \\
&  = &{{\bf{x}}^{\rm{T}}}{\bf{Hx}}.
\end{array}
\label{equ:cost}
\end{equation}

\subsubsection{Optimization problem}
The goal of the optimization problem is to minimize the energy cost meanwhile satisfying constraints. Thus, combining (\ref{equ:equ_cons}), (\ref{equ:con_equ}), and (\ref{equ:cost}), this problem could be expressed as 
\begin{equation}
\begin{array}{ll}
\mathop {\min }\limits_{\bf{x}} &{{{\bf{x}}^{\rm{T}}}{\bf{Hx}}}\\
{\rm{s}}{\rm{.t}}{\rm{. }}&{\bf{Ax}} = {\bf{b}}_k\\
&{f_i}\left( {\bf{x}} \right) \le 0,i = 1,...,{n_c},
\end{array}
\end{equation} 
which is in the form concerned with \emph{Lemma 2}. Therefore, the optimal trajectories ${\bf h}^*\left(\left( {{{\bf{q}}_{0,k}},{{\bf{q}}_{m,k}}} \right),t\right)$ are derived by solving $q$ optimization problems for $q$ order pairs $\{{\left( {{{\bf{q}}_{0,k}},{{\bf{q}}_{m,k}}} \right)}\}$.

\subsection{Constructing optimal virtual tube}
The waypoints for any order pair $\left( {{{\bf{q}}_{0}\left({\bm \theta}_j\right)},{{\bf{q}}_{m}\left({\bm \theta}_j\right)}} \right)$ could be expressed as
\[{{\bf{q}}_i}\left( {\bm{\theta }_j} \right) = \sum\limits_{k = 1}^q {{\theta _{kj}}{{\bf{q}}_{i,k}}} ,i = 0,1,2,...,m.\] And the normalized knots $\left\{ {{t_i}} \right\}$ are used to parameterize the waypoints. Thus, the linear equality constraints are expressed as 
\begin{equation}
\left[ {\begin{array}{*{20}{c}}
	{{{\bf{A}}_1}}\\
	{{{\bf{A}}_2}}\\
	{{{\bf{A}}_3}}
	\end{array}} \right]{\bf{x}} = \sum\limits_{k = 1}^q {{\theta _{kj}}{{\left[ {\begin{array}{*{20}{c}}
				{{{\bf{b}}_1}}\\
				{{{\bf{b}}_2}}\\
				{{{\bf{b}}_3}}
				\end{array}} \right]}_k}.}
\end{equation}

According to \emph{Theorem \ref{the:convex_hull}}, the linear virtual tube with convex hull terminals $\left( {{\cal C}_0,{\cal C}_1,{\bf f},{{\bf h}^*}} \right)$ is optimal, as shown in Step \textcircled{5} of Fig. \ref{fig:tubeplanningprocess}. In other words, for any start point ${{\bf{q}}_0}\left( {{{\bm{\theta }}_j}} \right)$, a goal point ${{\bf{q}}_{m}\left({\bm \theta}_j\right)}$ is assigned, and then the optimal trajectory ${{\bf h}^*}\left( {\left( {{{\bf{q}}_{0}\left({\bm \theta}_j\right)},{{\bf{q}}_{m}\left({\bm \theta}_j\right)}} \right),t} \right)$ is expressed as
\begin{equation}
{{\bf h}^*}\left( {\left( {{{\bf{q}}_{0}\left({\bm \theta}_j\right)},{{\bf{q}}_{m}\left({\bm \theta}_j\right)}} \right),t} \right) = \sum\limits_{k = 1}^q {{\theta _{kj}}{{\bf h}^*}\left( {\left( {{{\bf{q}}_{0,k}},{{\bf{q}}_{m,k}}} \right),t} \right)} ,
\label{equ:construct_optimal_virtual_tube}
\end{equation}
where $\sum\nolimits_{k = 1}^q {{\theta _{kj}}}  = 1,{\theta _{kj}} \ge 0.$

{
	For all robots at each $t'$, the desired positions are $\{{{\bf h}^*}\left( {\left( {{{\bf{q}}_{0}\left({\bm \theta}_j\right)},{{\bf{q}}_{m}\left({\bm \theta}_j\right)}} \right),t'} \right)\}\left(j\in \mathcal{M}\right)$ which are within the convex hull of $\left\{ {{{\bf{h}}^*}\left( {\left( {{{\bf{q}}_0}_{,k},{{\bf{q}}_{m,k}}} \right),t'} \right)} \right\}\left( {k = 1,2,...,q} \right)$. In other words, the swarm would approach a convex hull at each $t'$.  
}

{
	\begin{remark}
		The optimal virtual tube has similar effects to the formation constraint. The swarm approaches a desired convex hull or called graph $\mathcal{C}_t$ at each time $t$. However, the distance constraint of adjacent robots is not considered in tube planning as it is in the process of controller design, which leaves more flexibility for the formation of the swarm, compared with the formation constraint.
	\end{remark}
}
{
\subsection{Complexity analysis}
In this subsection, we analyze the time complexity of the proposed method. Suppose the $q$ optimal trajectories ${\bf h}^*\left(\left( {{{\bf{q}}_{0,k}},{{\bf{q}}_{m,k}}} \right),t\right)$ have been generated before. And according to (\ref{poly_nomial}), the number of parameters in trajectory ${\bf h}^*\left(\left( {{{\bf{q}}_{0,k}},{{\bf{q}}_{m,k}}} \right),t\right)$ is $n_t = \left(n+1\right)md$ where $n$ is the orders of trajectory, $m$ is the number of the segments and $d$ is the number of dimension. 
The operation of the convex combination in (\ref{equ:construct_optimal_virtual_tube}) only has some basic operations including addition and multiplication on the $n_t$-dimension vectors. Therefore, time complexity of our method is $O\left(n_t\right)$.
}

{
As for solving optimization problem to generate the optimal trajectory, the time complexity is both related to the number of iterations and the time complexity of each iteration\footnote{There is an another popular approach to analyze the complexity of optimization. The black-box model is used for the oracle complexity, which regards queries to the object function as oracles and considers the number of queries (\cite{bubeck2015convex}). However, it is not suitable for comparison with the convex combination operation.}. Furthermore, the number of iterations is related to the precision $\varepsilon$ and the dimension $n_t$, which could be expressed as $k\left(n_t,\varepsilon\right)$. Therefore, the time complexity of solving optimization problem is normally from $O\left(k\left(n_t,\varepsilon\right)n_t^2\right)$ to $O\left(k\left(n_t,\varepsilon\right)n_t^3\right)$ (\cite{Wright1997Primal}). For example, the best-known interior-point method algorithm require $O\left(k\left(n_t,\varepsilon\right)\right) = O\left(\sqrt {{n_t}} \log \left( {{1 \mathord{\left/
			{\vphantom {1 \varepsilon }} \right.
			\kern-\nulldelimiterspace} \varepsilon }} \right)\right) $ iterations and the complexity of each iteration is roughly $O\left(n_t^3\right)$. Thus, the complexity of the interior-point method is $O\left( {n_t^{3.5}\log \left( {{1 \mathord{\left/
				{\vphantom {1 \varepsilon }} \right.
				\kern-\nulldelimiterspace} \varepsilon }} \right)} \right)$.
Compared with solving optimization problem, our method effectively reduces the computations by removing the iteration and matrix operations.
}
\section{Model Predictive Control for Tracking Trajectories in an Optimal Virtual Tube}
{The optimal virtual tube is applicable to a wide range of control methods.
Specifically, the optimal virtual tube can provide the desired trajectories, which makes the swarm exhibit an effect similar to the formation control by using the optimal control methods such as Model Predictive Control (MPC). Or it can also provide desired velocity command to a swarm based on the control method such as Artificial Potential Field (APF) and Control Barrier Function (CBF) for smoother control. 
}

{To easily understand and verify the characteristics of the optimal virtual tube, an application in control for the optimal virtual tube is provided.}
A hierarchical approach is employed to enable swarm robotics to move within the optimal virtual tube, comprising two layers: \emph{optimal virtual tube planning} and \emph{trajectory tracking} for a robot swarm. First, the optimal virtual tube planning is proposed in the above section, which generates optimal trajectories for the swarm. Subsequently, each robot tracks its own trajectory while avoiding collisions with other robots. This section first introduces a robot model and then designs an MPC controller to realize trajectory tracking and conflict avoidance.

\subsection{Robot model}
The dynamic of a holonomic robot is described as a mass point model:
\begin{equation}
\begin{array}{l}
{\bf{\dot p}}_i = {\bf{v}}_i,\\
{\bf{\dot v}}_i = {{\bf{u}}_i},
\end{array}
\label{equ:drone-model}
\end{equation}
where ${\bf p}_i \in  \mathbb{R}^{2}\left(\mathbb{R}^3\right)$ and ${\bf v}_i \in \mathbb{R}^{2}\left(\mathbb{R}^3\right)$ are the position and velocity of the $i$th robot, ${\bf u}_i \in \mathbb{R}^{2}\left(\mathbb{R}^3\right)$ is the acceleration input to the $i$th robot, $i=1,2,...,M$. 

This system (\ref{equ:drone-model}) could be expressed as 
\begin{equation}
{\bf{\dot x}} = {\bf{Ax}} + {\bf{Bu}}
\end{equation}
where ${\bf{x}} = {\left[ {\begin{array}{*{20}{c}}
		{{{\bf{p}}_i}}&{{{\bf{v}}_i}}
\end{array}} \right]^{\rm{T}}}$, ${\bf{u}} = {{\bf{u}}_i}$. And the discretized form is
\begin{equation}
{{\bf{x}}_{k + 1}} = {{\bf{A}}_k}{{\bf{x}}_k} + {{\bf{B}}_k}{{\bf{u}}_k}
\end{equation}
where ${\bf{x}}_k$ and ${\bf{u}}_k$ are state and input at the time $k$, \[{{\bf{A}}_k} = \left[ {\begin{array}{*{20}{c}}
	{\bf{I}}&{{T_s}{\bf{I}}}\\
	{\bf{0}}&{\bf{0}}
	\end{array}} \right],{{\bf{B}}_k} = \left[ {\begin{array}{*{20}{c}}
	{\bf{0}}\\
	{{T_s}{\bf{I}}}
	\end{array}} \right].\]
\subsection{Dynamic obstacles}
As depicted in Fig. \ref{fig:obstacledronecons}(a), for a robot $i$ in the swarm, other robots $j$ is regarded as the dynamic obstacles which are represented by an ellipse of area ${\mathcal{O}_{j,k}}$ at time $k$. The ellipse of area ${\mathcal{O}_{j,k}}$ at time $k$ is denoted by
\begin{equation}
{{\cal O}_{j,k}} = \left\{ {{\bf{p}}|\left\| {{\bf{E}}\left( {{\bf{p}} - {{\bf{p}}_{j,k}}} \right)} \right\| \le 1} \right\},k = 0,1,...,N,
\end{equation}
where ${\bf p}_{j,k}$ is the position of the center mass of robot $j$ at time $k$, ${\bf E}={\rm diag}\left(a_j,b_j,c_j\right)$ is an invertible scaling matrix to bound the boundary of the ellipse. The details of designing ${\bf E}$ can be found in \cite{Quan2021far}. If the robot $i$ is regarded as a mass point, the obstacle space ${{\cal O}_{k}}$ of robot $i$ at time $k$ is expressed as
\begin{equation}
{{\cal O}_k} = \bigcup\limits_{j \in {\cal M},j \ne i} {{{\cal O}_{ij,k}},} 
\end{equation}
where ${{\cal O}_{ij,k}} = \left\{ {{\bf{p}}|{\bf{p}} \in {{\cal O}_{i,k}} \oplus {{\cal O}_{j,k}}} \right\}$, $\oplus$ is the Minkowski sum. However, when the robot $i$ does not collide with other robots, the avoidance constraints 
\begin{equation}
{{\bf{p}}_{i,k}} \cap {{\cal O}_{k}} = \emptyset
\label{equ:non-cov-avoid-cons}
\end{equation}
are not convex constraints. Thus, half spaces ${{\cal H}_{ij,k}}$, as shown in Fig. \ref{fig:obstacledronecons}(b), are used to construct affine inequality constraints at each time $k$. The constraints (\ref{equ:non-cov-avoid-cons}) are transformed into
\begin{equation}
{\cal O} = \left\{ {{{\bf{x}}_k}|{{\bf{H}}_k}{{\bf{x}}_k} \le {h_k} + {s_k},k = 0,1,...,N} \right\},
\end{equation}
where ${\bf x}_k$ are states of robot $i$, ${\bf{H}}_k$, ${h_k}$ are parameters of half space constraints, $s_k$ is a positive relaxation variable to avoid infeasibility problem in practice.
\begin{figure}
	\centering
	\includegraphics{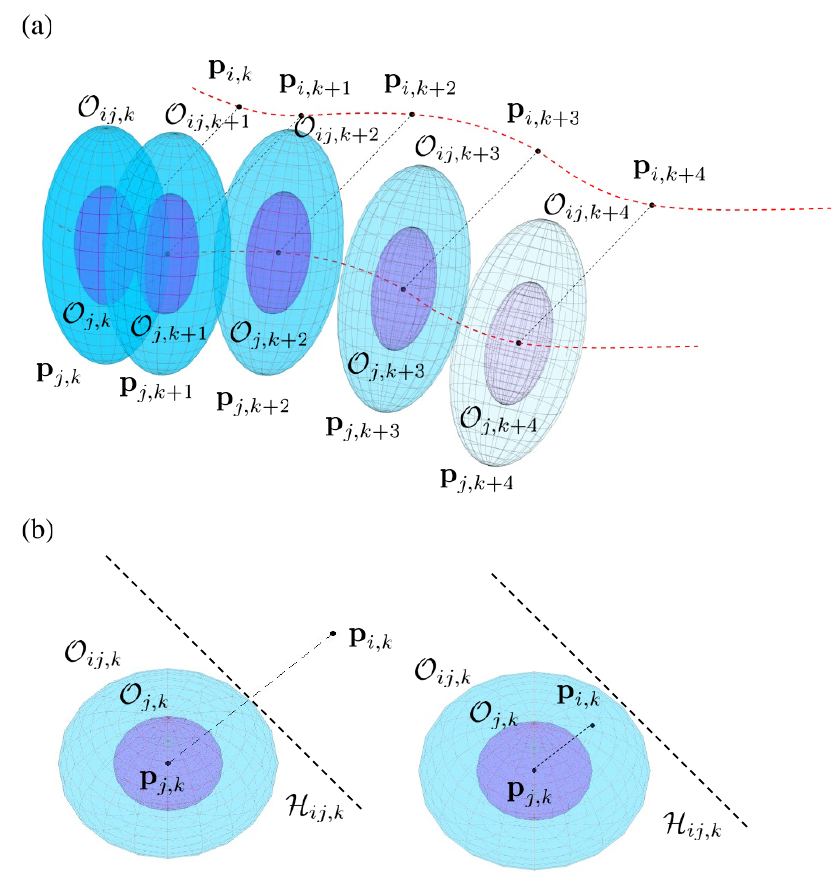}
%
	\caption{Constructing constraints for avoiding collisions. The purple ellipses ${{\cal O}_{j,k+n}}\left(n=0,1,...,4\right)$ are boundaries of robot $j$  at time $k+n$. The blue ellipses ${{\cal O}_{ij,k+n}}\left(n=0,1,...,4\right)$ are obstacle space for robot $i$ at time $k+n$. (a) The obstacle constraints for robot $i$ at time $k+n$. (b) Two situations: collision and avoidance. The dotted lines are tangent planes ${{\cal H}_{ij,k}}$ of points which are intersections of surfaces of blue ellipses and connections of points ${{\bf{p}}_{j,k}},{{\bf{p}}_{i,k}}$. }
	\label{fig:obstacledronecons}
\end{figure}

\subsection{Optimization problem}
The distance to desired state ${{\bf{x}}_{d,k}}={\left[ {\begin{array}{*{20}{c}}
		{{{\bf{p}}_{d,k}}}&{{{\bf{v}}_{d,k}}}
		\end{array}} \right]^{\rm{T}}}$ and desired input ${\bf u}_{d,k}$ at time $k$ generated by desired trajectory ${{\bf{h}}^*}\left( {\left( {{{\bf{q}}_0},{{\bf{q}}_m}} \right),t\left( s \right)} \right)$\footnote{In practice, the normalized knots $\{t_i\}$ are not suitable for a robot to track trajectory. To address this issue, the $t\left( s \right) = \frac{{{v_R}}}{{{u_m}}}s$ is used in the trajectory ${{\bf{h}}^*}\left( {\left( {{{\bf{q}}_0},{{\bf{q}}_m}} \right),t\left( s \right)} \right)$, where $v_R$ is a coefficient related with robots, $u_m\in \{u_i\}$.} is minimized so that the robot could track the desired trajectory.  Thus, the model predictive control problem containing robot model and other constraints in $N+1$ steps is formulated as
\begin{equation}
\begin{array}{ll}
\mathop {\min }\limits_{{{\bf{x}}_k},{{\bf{u}}_k}} & {\bf{\tilde x}}_N^{\rm{T}}{{\bf{Q}}_{x,N}}{{{\bf{\tilde x}}}_N} + \sum\limits_{k = 0}^N {{q_{s,k}}s_k^2}  + \sum\limits_{k = 0}^{N - 1} {{\bf{\tilde x}}_k^{\rm{T}}{{\bf{Q}}_{x,k}}{{{\bf{\tilde x}}}_k}} + \\ &{\bf{\tilde u}}_k^{\rm{T}}{{\bf{Q}}_{u,k}}{{{\bf{\tilde u}}}_k} \\
{\rm{s}}{\rm{.t}}{\rm{.  }}&\\
&{{\bf{x}}_{k + 1}} = {{\bf{A}}_k}{{\bf{x}}_k} + {{\bf{B}}_k}{{\bf{u}}_k}\\
&{{\bf{x}}_k} \in {\cal X}_{\rm t} \cap {\cal O},\\
&{{\bf{u}}_k} \in {\cal U},\\
&k = 0,1,...,N,
\end{array}
\end{equation}
where ${{{\bf{\tilde x}}}_k} = {{\bf{x}}_{d,k}} - {{\bf{x}}_k}$, ${{\bf{\tilde u}}_k} = {{\bf{u}}_{d.k}} - {{\bf{u}}_k}$, ${{\bf{Q}}_{x,k}}$ and ${{\bf{Q}}_{u,k}}$ are positive definite coefficient matrixs, ${\cal X}_{\rm t}$ is the feasible set of ${\bf x}_k$, ${\cal U}$ is the feasible set of ${\bf u}_k$. 
{The feasible set of ${\bf v}_k$ in ${\mathcal{X}_{\rm t}}$ satisfies the dynamic constraints of robots. And the feasible set of ${\bf p}_k$ in ${\mathcal{X}_{\rm t}}$ is the interior of the virtual tube ${\bf{int}}\left( {\cal T} \right)$ for the robot within the virtual tube, and $\left\{ {{{\bf{p}}_k}|\left\| {{{\bf{p}}_k} - {{\bf{p}}_{d,k}}} \right\| \le \varepsilon_c  \in {\mathbb{R}^ + }} \right\}$ for the robot on the boundary of the virtual tube $\partial {\cal T}$ where $\varepsilon_c$ is a feasible constant.} 
\section{Simulations and Experiments}
This section presents simulations and experiments of the drone swarm in simple scenarios to demonstrate the effectiveness of the proposed optimal virtual tube planning method and the MPC controller. 
{Comparisons with existing methods including multi-robots formation control (\cite{alonso2017multi}) and predictive control of aerial swarm (\cite{soria2021predictive}) are made in complex scenarios. }
The optimal virtual tube planning and trajectory tracking in simulations are implemented using MATLAB code, and executed on a PC with Intel Core i7-7700 @ 2.8GHz CPU and 16G RAM.
The safety distance between drones is $1$m. A video of simulations and experiments is available on \url{https://youtu.be/9pT5SiCsZis}.
\subsection{Simulations in simple scenarios}  
The results of the optimal virtual tube planning and trajectory tracking are obtained. To demonstrate the effectiveness of the proposed method, comparisons between the optimal virtual tube planning and a traditional planning method are made in 2-D and 3-D space. The simulation results demonstrate that the proposed method with a computational complexity $O\left(n_t\right)$ produces the same results as the traditional planning method with a computational complexity from $O\left(k\left(n_t,\varepsilon\right)n_t^2\right)$ to $O\left(k\left(n_t,\varepsilon\right)n_t^3\right)$ while requiring less computational effort.
\subsubsection{A 2-D optimal virtual tube}
The optimal virtual tube planning is first validated in a 2-D map. Two methods were set with identical terminals ${\mathcal{C}_0}$ and $\mathcal{C}_1$, as illustrated in Fig. \ref{fig:compare_2d} with dotted lines. Additionally, 11 start positions and 11 goal positions are also the same. For the traditional planning method, 11 separate optimization problems are required to be solved, as depicted in Fig. \ref{fig:compare_2d}(a). In contrast, for the proposed method, only two trajectories need to be planned: black curves in Fig. \ref{fig:compare_2d}(b). The remaining 9 trajectories, depicted by red curves in Fig. \ref{fig:compare_2d}(b), are generated by interpolation of the two planned trajectories. Similarly, 101 trajectories are generated in the same way, requiring the resolution of only two optimization problems for two trajectories. Comparing the traditional method with the proposed method, the errors were found to be negligible, with magnitudes less than $1.8\times10^{-14}$m, as shown in Fig. \ref{fig:tubeplan2derr}. Thus, this outcome validates \emph{Theorem \ref{the:convex_hull}}.
\begin{figure*}
	\centering
	\includegraphics{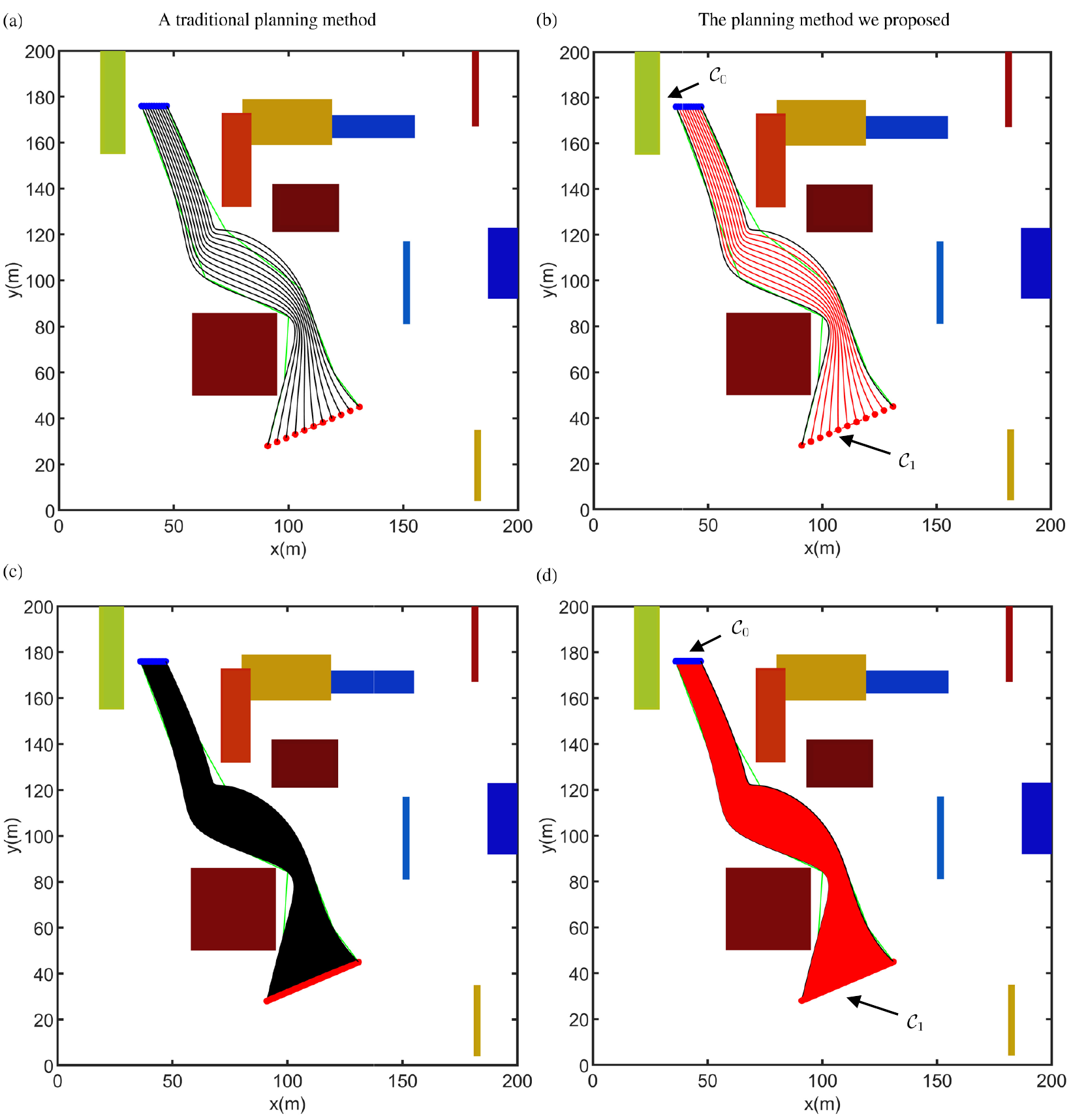}
%
	\caption{Comparison between the traditional method and the optimal virtual tube planning. (a) Trajectories planned by 11 different optimization problems with the traditional method. (b) Trajectories are planned only by two optimization problems with optimal virtual tube planning. (c) Trajectories planned by 101 different optimization problems with the traditional method. (d) Trajectories are planned only by two optimization problems with optimal virtual tube planning. The colorful rectangles are obstacles. The blue points and the red points are start points and goal points. The optimal trajectories planned by optimization problems are depicted in black, while the optimal trajectories obtained by interpolation are shown in red. The magenta dotted lines denote the terminals.}
	\label{fig:compare_2d}
\end{figure*}
\begin{figure*}
	\centering
	\includegraphics{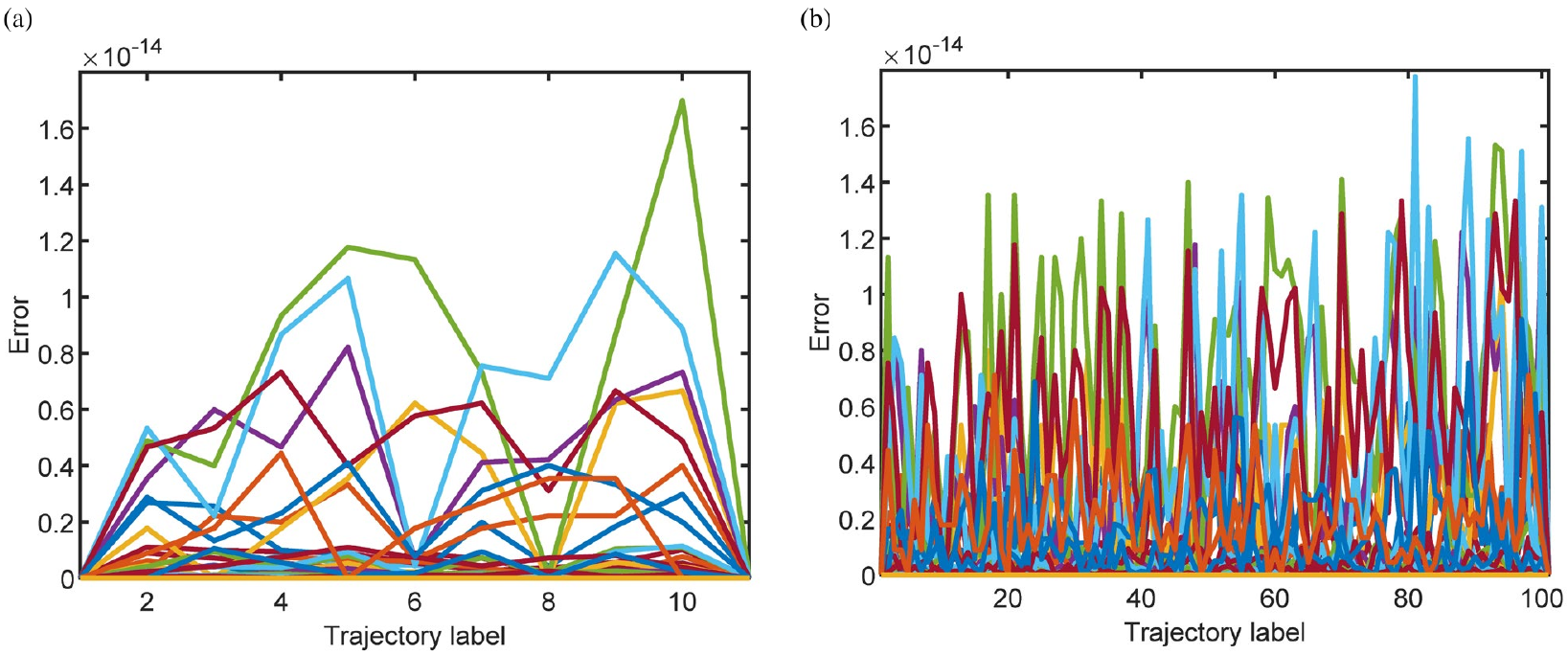}
	\caption{Errors between variable vectors by interpolation and optimization program are illustrated using colorful curves representing the error in each parameter of the variable vectors. (a) The errors of 11 trajectories. (b) The errors of 101 trajectories.}
	\label{fig:tubeplan2derr}
\end{figure*}
To demonstrate that the error is independent of the number of interpolations, the variable vector is analyzed and the distribution of one parameter in the variable vector is shown with respect to the number of trajectories obtained by interpolation in Fig. \ref{fig:error2d}. The results show that the error for this parameter remains stable as the number of interpolations increases from 10 to 1000.
\begin{figure}
	\centering
	\includegraphics{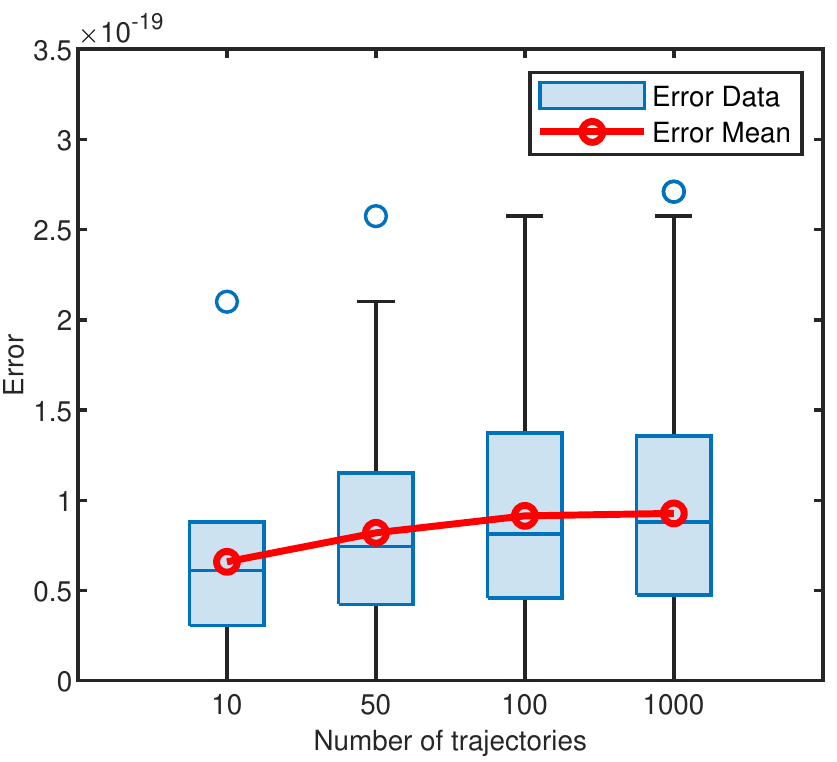}
	\caption{Error distribution of a parameter in variable vector.}
	\label{fig:error2d}
\end{figure}

In Fig. \ref{fig:compare_2d}(b), the terminal $\mathcal{C}_0$ is represented by a line with a length of $10m$. To validate the controller, 11 drones are tightly placed in the terminal $\mathcal{C}_0$, as shown in Fig. \ref{fig:follow2d}. Then, 11 optimal trajectories are assigned to each drone. In the result of the simulation depicted in Fig. \ref{fig:follow2d}, the drone swarm safely and smoothly arrives at terminal $\mathcal{C}_1$. When the drones are far apart, they follow their own trajectories. However, when the swarm passes through a narrow cross-section of the virtual tube in Fig. \ref{fig:follow2d}, some drones need to avoid nearby drones, which may cause them to deviate from their trajectories ($t = 12.96$s) but return to planned trajectories after passing through the narrow space ($t = 14.80$s). The minimum distance among drones in the swarm with respect to the time shown in Fig. \ref{fig:follow2dmindis} is always greater than the safety distance $1m$, which indicates that there is no collision among drones. To demonstrate the accuracy of drone tracking, errors between the desired and true trajectory of all drones are depicted in Fig. \ref{fig:follow2derr}. Therefore, the simulation demonstrates that every drone can track the trajectory with high accuracy while avoiding other drones.
\begin{figure*}
	\centering
	\includegraphics{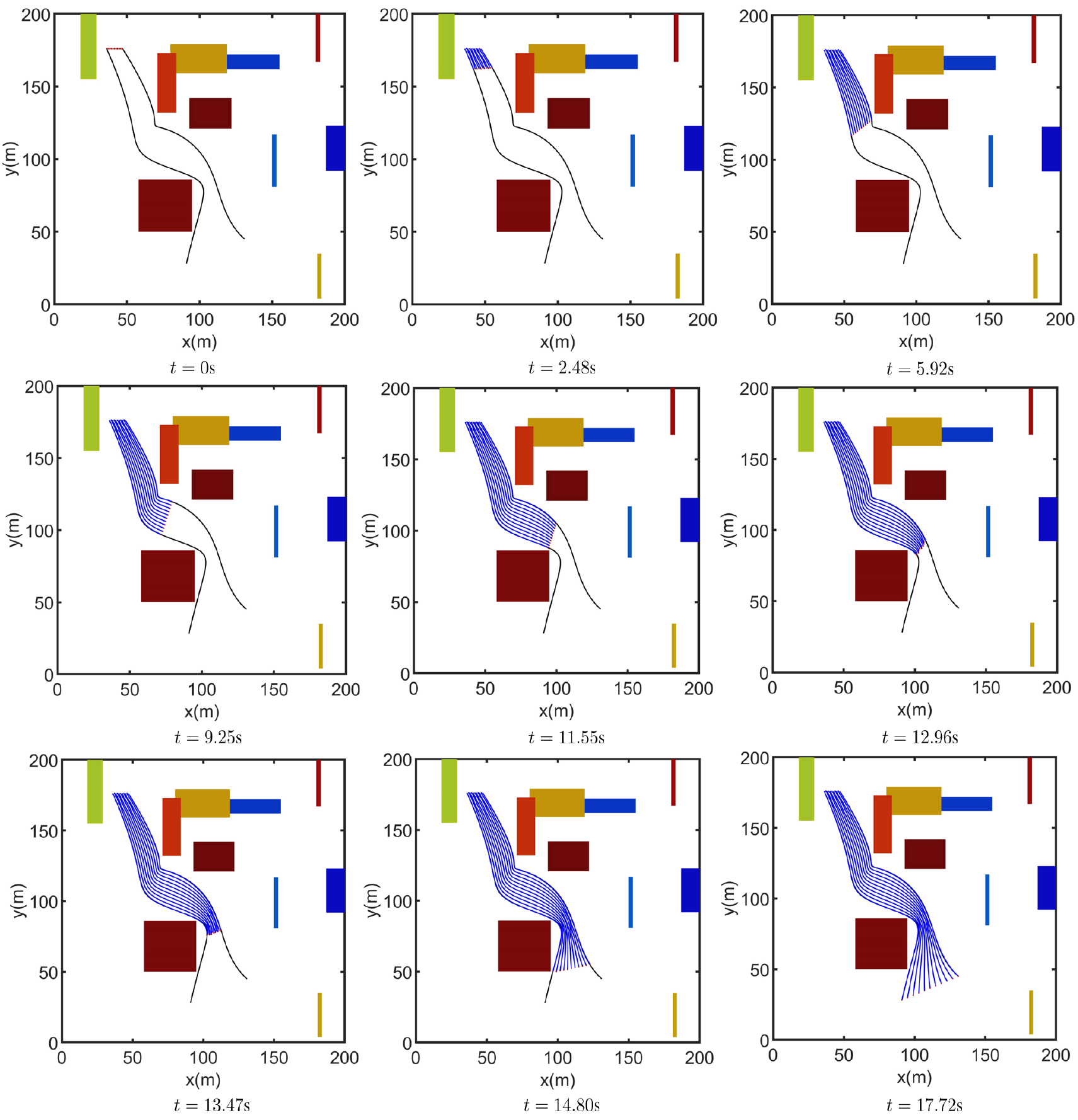}
%
%
	\caption{Simulation results of drone swarm trajectory tracking. The 11 dark blue lines are trajectories for 11 drones, while the red circles indicate drones.}
	\label{fig:follow2d}
\end{figure*}
\begin{figure}
	\centering
	\includegraphics[width=0.98\linewidth]{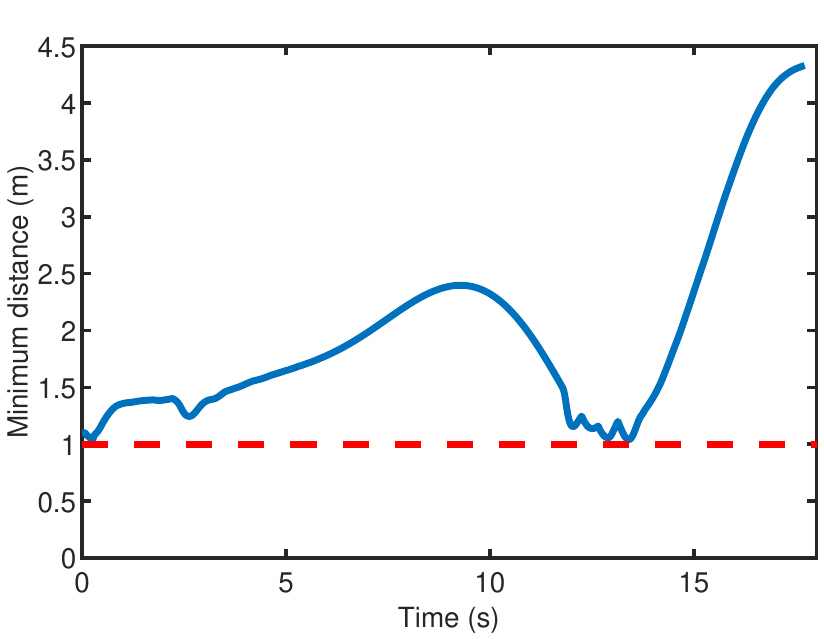}
	\caption{Minimum distance among drones. The blue curve is the minimum distance with respect to time. The red dotted line is the safety distance among drones.}
	\label{fig:follow2dmindis}
\end{figure}
\begin{figure}
	\centering
	\includegraphics{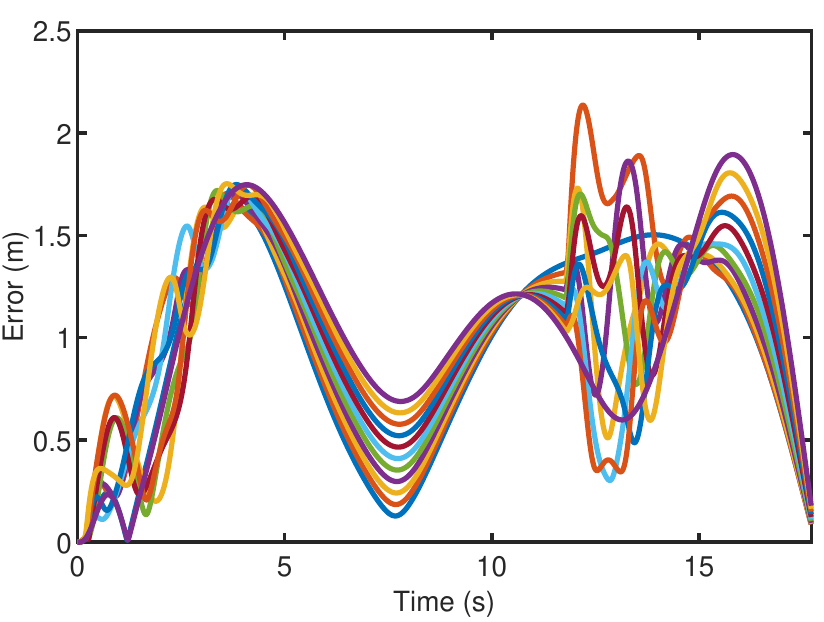}
	\caption{Comparison of actual and desired trajectories for the drone swarm.}
	\label{fig:follow2derr}
\end{figure}

\subsubsection{A 3-D optimal virtual tube}
For a 240m$ \times $240m$ \times $120m map, illustrated in Fig. \ref{fig:process_of_tube_planning}(a), the start terminal $\mathcal{C}_0$ which is the convex hull of the set $\{{\bf q}_{0,k}\}$ is represented as a green tetrahedron, while the goal terminal $\mathcal{C}_1$ which is the convex hull of the set $\{{\bf q}_{m,k}\}$ is represented as a red tetrahedron. First, the order pairs $\left( {\bf q}_{0,k}, {\bf q}_{m,k}\right) $ are constructed by assigning the red points to green points respectively. Thus, any start point ${\bf q}_{0}\left({\bm \theta}_j\right)$ in terminal ${\mathcal{C}_0}$ is expressed as (\ref{equ:start_position}), and the map ${\bf f}$ maps this point to the goal point ${\bf q}_{m}\left({\bm \theta}_j\right)$. Then, a set of order pairs ${\mathcal{P}}$ is constructed as $\left\{ {\left( {{{\bf{q}}_{0}\left({\bm \theta}_j\right)},{{\bf{q}}_{m}\left({\bm \theta}_j\right)}} \right)} \right\}$. Next, the paths of the order pairs ${\left( {{{\bf{q}}_{0,k}},{{\bf{q}}_{m,k}}} \right)}$ are generated by a tree-based motion planner called RRT*, as shown in Fig. \ref{fig:process_of_tube_planning}(b). The optimal trajectories ${\bf h}^*$ of the order pairs ${\left( {{{\bf{q}}_{0,k}},{{\bf{q}}_{m,k}}} \right)}$ are then obtained by solving optimization problems, as depicted in Fig. \ref{fig:process_of_tube_planning}(c). Finally, the optimal virtual tube is generated based on \emph{Theorem \ref{the:convex_hull}}, which is illustrated in Fig. \ref{fig:process_of_tube_planning}(d).
\begin{figure*}
	\centering
	\includegraphics{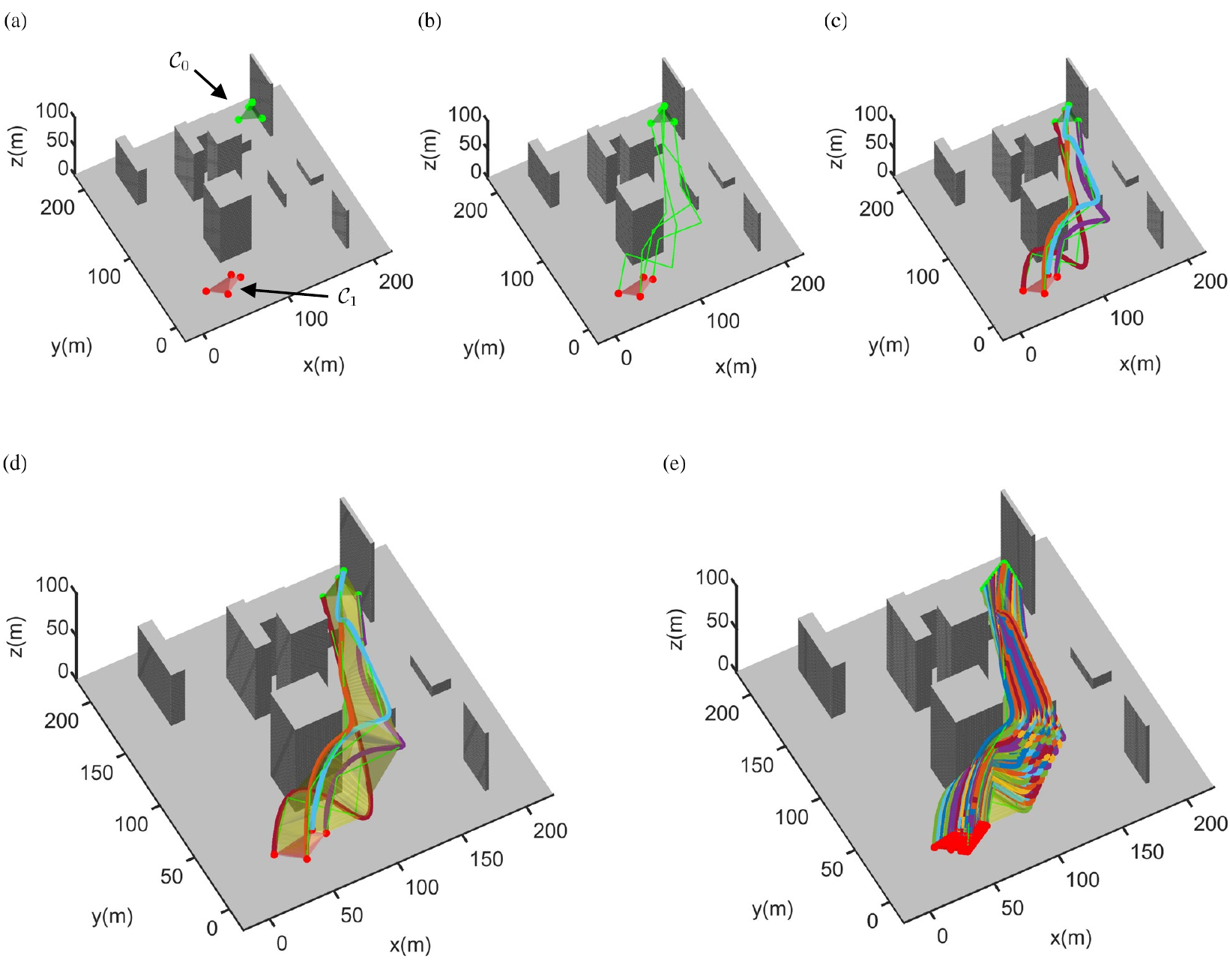}
%
	
	\caption{An optimal virtual tube planning process. (a) Constructing order pairs ${\mathcal{P}}$. The red points represent the start points, and the green points represent the goal points. The red trigonal pyramid denotes the start area, and the green trigonal pyramid denotes the goal area. (b) Finding paths. A path-finding algorithm is used to find paths, shown as green lines. (c) Optimizing trajectories. Optimal trajectories are generated by solving optimization problems. The colorful lines represent different trajectories for different pairs. (d) Constructing optimal virtual tube. An optimal virtual tube is generated based on the optimal trajectories, shown as a yellow area. (e) Generating infinite optimal paths. Other optimal trajectories, denoted by colorful lines, are automatically generated by interpolation.}
	\label{fig:process_of_tube_planning}
\end{figure*}

The drone swarm comprises 84 drones, and their start and goal points are generated by equispaced interpolation. Subsequently, the 84 optimal trajectories for all drones are generated by interpolation between four optimal trajectories. The error between the trajectories by interpolation and by solving 84 optimization problems, as shown in Fig. \ref{fig:map_tube_trajectory}, is negligible. Then, drones could track their trajectories correspondingly. 
\begin{figure}
	\centering
	\includegraphics{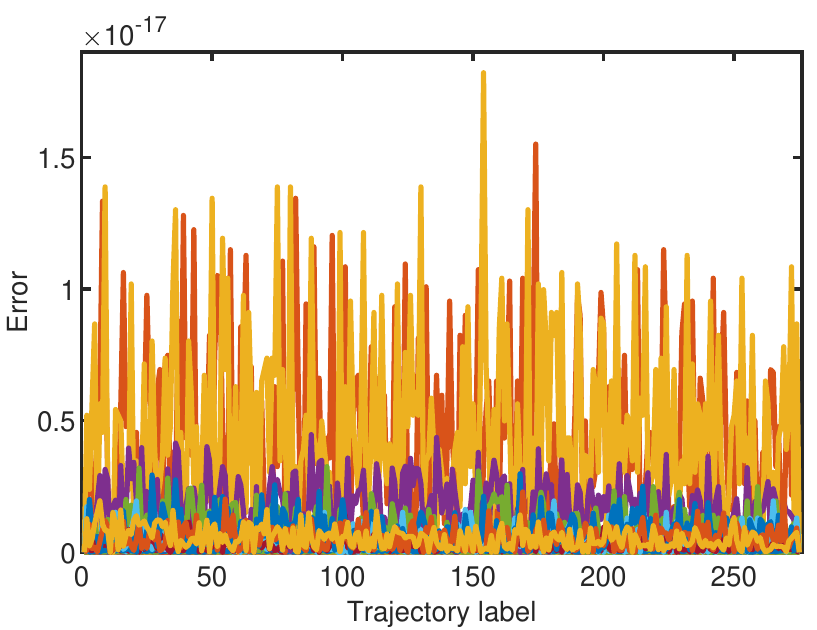}
	\caption{Errors between variable vectors of trajectories by interpolation and optimization problem are illustrated using colorful lines representing the error in each parameter of the variable vectors.}
	\label{fig:map_tube_trajectory}
\end{figure}

The terminals $\mathcal{C}_0$ and $\mathcal{C}_1$ in Fig. \ref{fig:process_of_tube_planning}(a) are tetrahedra. The drone swarm with 84 drones is placed in the terminal $\mathcal{C}_0$. There are only 4 trajectories being planned and other 80 trajectories are generated by interpolation. In the simulation result, as shown in Fig. \ref{fig:follow3d}, the drone swarm arrives at terminal $\mathcal{C}_1$ safely and smoothly. The minimum distance among drones in the swarm with respect to the time shown in Fig. \ref{fig:follow3dmindis} is always higher than the safety distance, which demonstrates no collision among drones.
\begin{figure}
	\centering
	\includegraphics{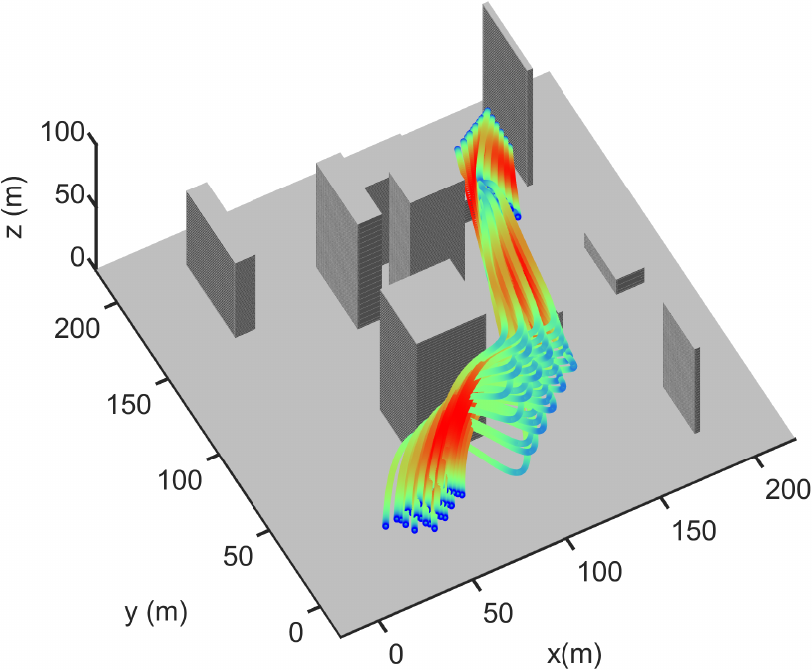}
	\caption{Simulation results of drone swarm following in 3-D. The colorful curves are different trajectories. The colors of trajectories correspond to varied speeds where the red represents the high speed and the blue represents the low speed.}
	\label{fig:follow3d}
\end{figure}
\begin{figure}
	\centering
	\includegraphics{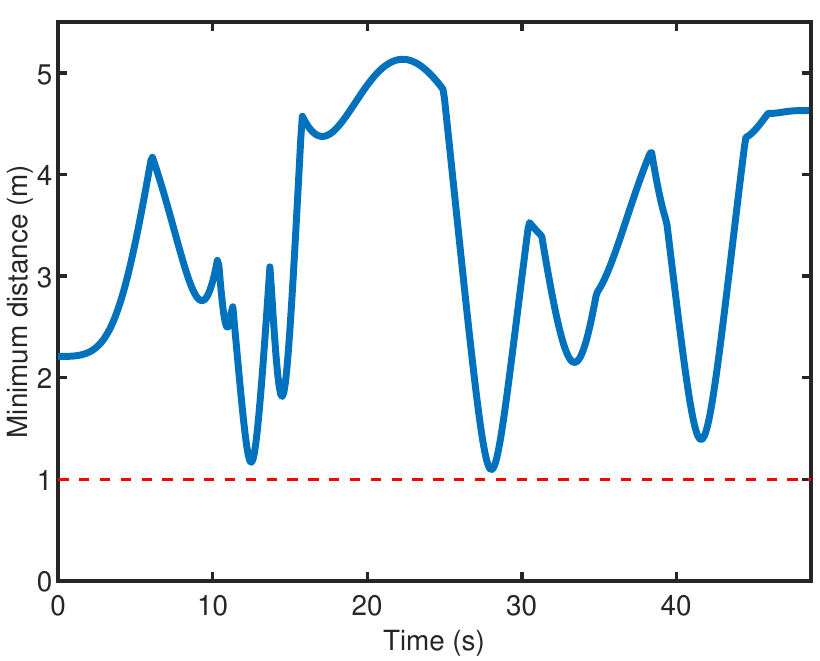}
	\caption{Minimum distance among drones in 3-D. The red dotted line is the $1$m safety distance.}
	\label{fig:follow3dmindis}
\end{figure}
\begin{figure}
	\centering
	\includegraphics{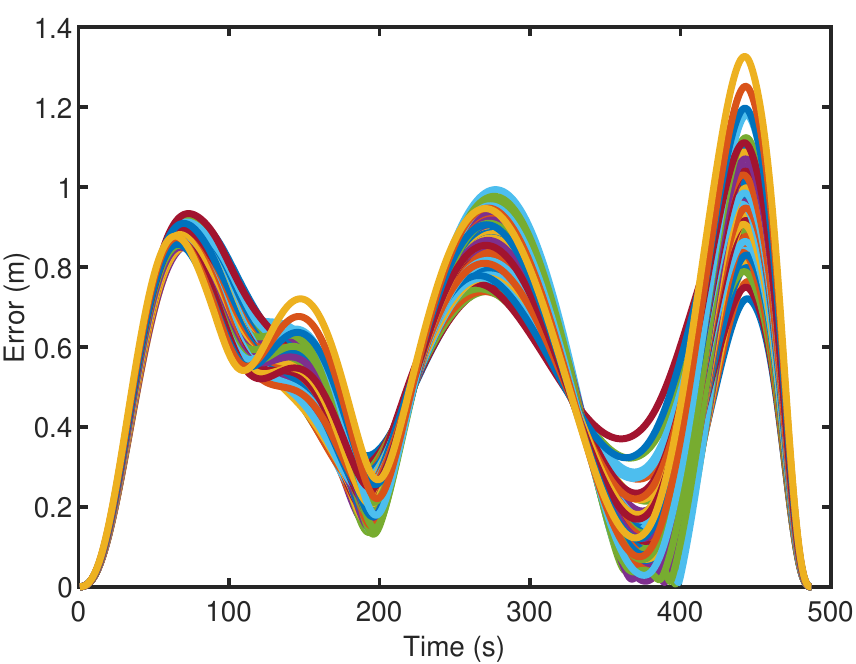}
	\caption{Error between desired and actual trajectory for drone swarm.}
	\label{fig:follow3derr}
\end{figure}
{
\subsection{Comparisons in complex scenarios}
In this subsection, we assess the performance of the proposed method in terms of navigation efficiency in a complex scenario. We first present several performance metrics and then compare different methods in complex scenarios.
}
{
\subsubsection{Performance metrics}
We use the following performance metrics to compare our proposed method with other methods. 
\begin{itemize}
	\item Average time: The average of the arrival times of robots in the swarm to the goal area. If any robot in the swarm does not reach the goal area within a certain time limit, the average time is set to be infinite.
	\item Arrival rate: The ratio of the number of robots reaching the goal area to the total number of robots in a certain time limit.
	\item Average speed: The average speed of the robots reaching the goal area in a certain time limit. The other robots not reaching the goal area are not considered.
\end{itemize}
We run each method 30 times for each scenario to report the mean for each metric.
}
{
\subsubsection{Complex scenarios}
In the complex scenarios, as shown in Fig. \ref{fig:complexscenarios}, the obstacles are randomly placed and the start area and goal area are designed as the tetrahedra. We test the navigation performance of the swarm in scenarios with 10, 20, 30, 40, and 50 random obstacles each. And the numbers of drones are set as the same.
}
\begin{figure}
	\centering
	\includegraphics{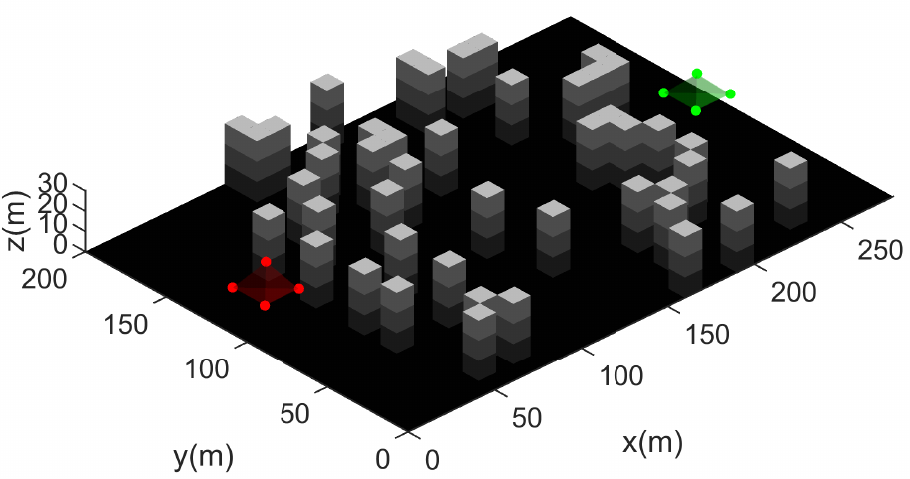}
	\caption{Complex scenario with random obstacles. The start area and goal area are denoted by the red and green tetrahedra.}
	\label{fig:complexscenarios}
\end{figure}
\begin{table*}[]
	\centering
	\begin{tabular}{clccc}
		\hline
		\multicolumn{1}{l}{\multirow{2}{*}{The number of obstacles}} & \multirow{2}{*}{Metrics} & \multicolumn{3}{c}{Methods}                                   \\ \cline{3-5} 
		\multicolumn{1}{l}{}                                         &                          & Formation control & Optimal virtual tube & Predictive control \\ \hline
		\multirow{3}{*}{10}                                          & Average time (s)         &   67.8                & 58.6                     &  \textbf{53.1}                  \\
		& Arrival rate (\%)        &  100                 &  100                    &   100                 \\
		& Average speed (m/s)      &  4.4556                 & \textbf{4.6656}                     &   4.516                \\ \hline
		\multirow{3}{*}{20}                                          & Average time (s)         &   67.5                & 59.3                    &  \textbf{51.5}                  \\
		& Arrival rate (\%)        &  100                 &  100                    &   100                 \\
		& Average speed (m/s)      &  4.5032                 & \textbf{4.5992}                     &   4.5871                \\ \hline
		\multirow{3}{*}{30}                                          & Average time (s)         &   84.6                &   \textbf{76.0}                   & $\infty$                   \\
		& Arrival rate (\%)        & 100                  & 100                     &  35                  \\
		& Average speed (m/s)      &  4.4792                 & \textbf{4.6083}                     &  1.9498                  \\ \hline
		\multirow{3}{*}{40}                                          & Average time (s)         &   63.2                & \textbf{57.0  }                  &  $\infty$                  \\
		& Arrival rate (\%)        &  100                 &  100                    &   75                 \\
		& Average speed (m/s)      &  4.5623                 & \textbf{4.8729}                     &   3.4681                \\ \hline
		\multirow{3}{*}{50}                                          & Average time (s)         &  83.3                 &  \textbf{65.3 }                   &  $\infty$                \\
		& Arrival rate (\%)        &    100               &  100                    &    75                \\
		& Average speed (m/s)      &  4.6201                 &   \textbf{4.8541}                   &   4.2050                \\ \hline
	\end{tabular}
	\caption{Performance metrics for different methods in complex scenarios with a varied number of obstacles.}
	\label{table:comparison}
\end{table*}
\begin{figure*}
	\centering
	\includegraphics{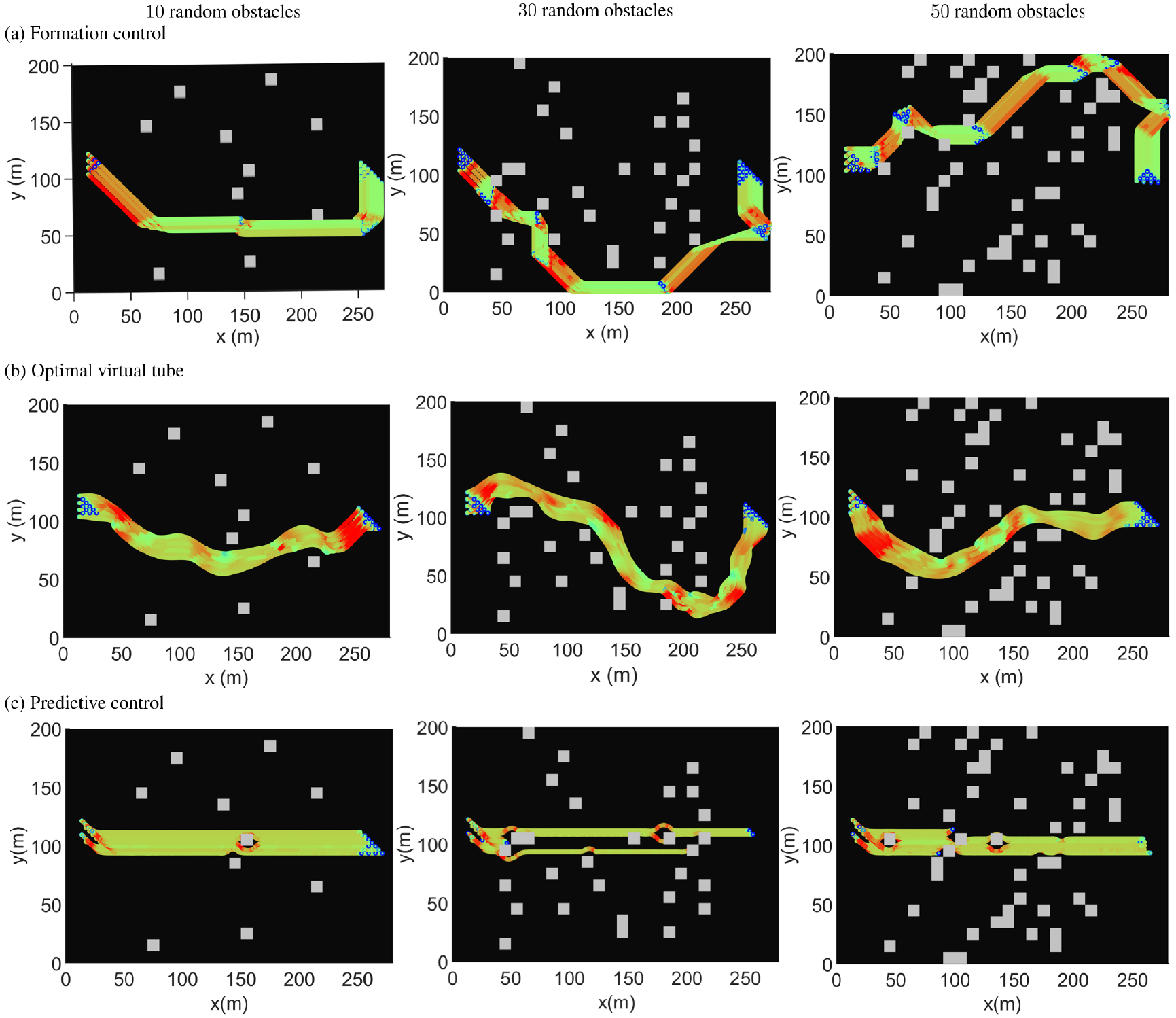}
	\caption{Trajectories of swarm robotics for three methods in environments with different numbers of obstacles. The colors of trajectories correspond to varied speeds where the red represents the high speed and the blue represents the low speed. }
	\label{fig:comparison}
\end{figure*}

{
Let the radius of drones be $0.5$m and the maximum speed of drones be $5$m/s. 
The results of the comparisons among the three methods are summarized in Table \ref{table:comparison}. Three typical scenarios with 10, 30, and 50 random obstacles are shown in Fig. \ref{fig:comparison}.
(i) In an optimal virtual tube, each drone is assigned to an optimal trajectory considering the kinematics of the drone. Compared with other methods, the desired optimal trajectories make drone swarm fly more smoothly and fast as a group as shown in Fig. \ref{fig:comparison}(b). 
Constraints in the obstacle avoidance process are efficiently simplified as the result of replacing the normal constraints with tube boundary constraints. Moreover, there is no intersection between the optimal trajectories for all drones in the optimal virtual tube so that the collision avoidance among drones is reduced. 
Therefore, the optimal virtual tube planning method is superior to other methods in terms of average time and average speed in obstacle-dense environments. 
Meanwhile, the optimal virtual tube has the same arrival rate as the formation control due to containing goal path planning algorithms in both methods to avoid falling into a deadlock. Compared with the formation control, the optimal virtual tube achieves fast passing through obstacle-dense environments by relaxing hard formation constraints. The swarm will try to maintain the desired formation during the flight but will deform in a small gap between obstacles.
(ii) As shown in Fig. \ref{fig:comparison}(a), the multi-robot formation control guarantees that all drones in the swarm arrive at the goal area while maintaining formation. The global path planner used in formation control finds several intersecting convex polytopes to remain the swarm in the free space so that all drones could reach the goal area, namely, 100\% arrival rate can be achieved. 
However, the formation constraint also requires a minimum volume of the convex polytope while maintaining the formation, resulting in a longer planned path, compared with other methods. The average time, thus, is longer than that of our method.
(iii) As for the predictive control of the aerial swarm, it has the minimum arrival time in obstacle-sparse environments.
However, it should be noted that the average time is set to be infinity in obstacle-dense environments due to at least one drone does not reach the goal area in a certain time limit.
Without the goal paths, the predictive control considers the goal point as the direction of migration that has a better performance in obstacles-sparse environments. 
However, in obstacles-dense environments especially with non-convex obstacles, the direction of obstacle avoidance may be opposite to the direction of migration so that the drone is easily blocked resulting in the lowest arrival rate, as shown in Fig. \ref{fig:comparison}(c). And the several non-regular obstacle constraints negatively affect the average speed. 
}
\subsubsection{Summaries}
{
	In this subsection, we have validated the superiority of our method in several scenarios with random obstacles by comparing with the state-of-art methods including multi-robot formation control and predictive control of aerial swarm.
	The metrics of our method including arrival rate, average time, and average speed are superior to other methods in obstacle-dense environments. The swarm could pass through obstacle-dense environments as a group fast and smoothly.
}
\subsection{Experiments}
A 30m$\times$10m$\times$8m space for real flight is depicted in Fig. \ref{fig:shanghaispace}. Known obstacles are randomly positioned on the ground within this area, and a VICON motion capture system is utilized for drone localization. The drone depicted in Fig. \ref{fig:dronepic}, which has a wheelbase of 350 millimeters, is equipped with a Jetson NX for onboard computation and utilizes the CUAV V5 nano Autopilot as its flight controller.
\begin{figure}
	\centering
	\includegraphics[width=0.99\linewidth]{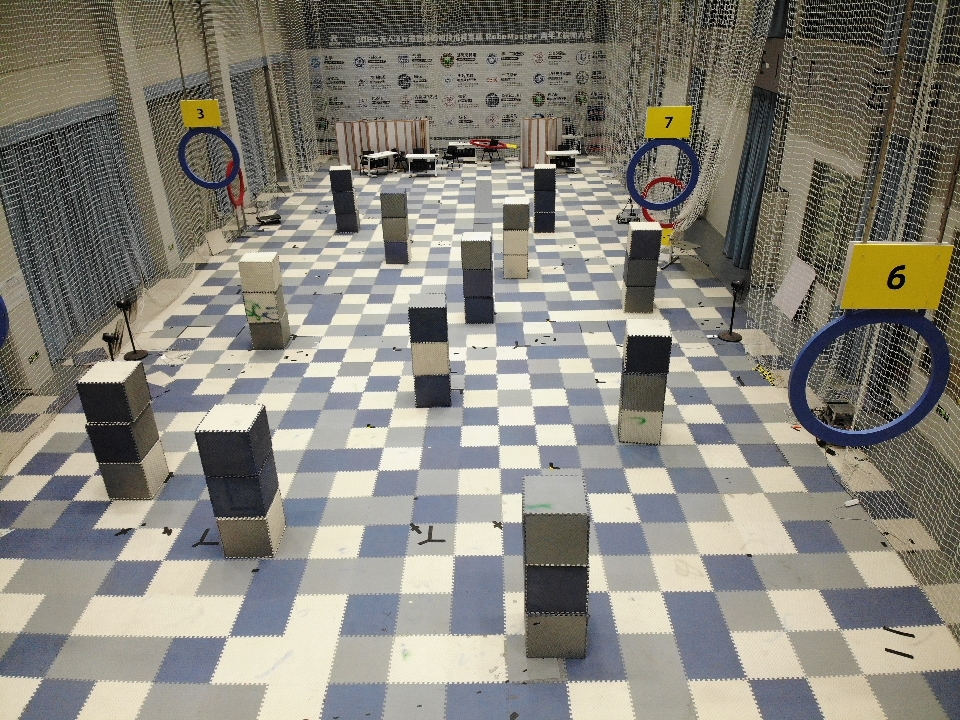}
	\caption{Space for real flight.}
	\label{fig:shanghaispace}
\end{figure}
\begin{figure}
	\centering
	\includegraphics[width=0.7\linewidth]{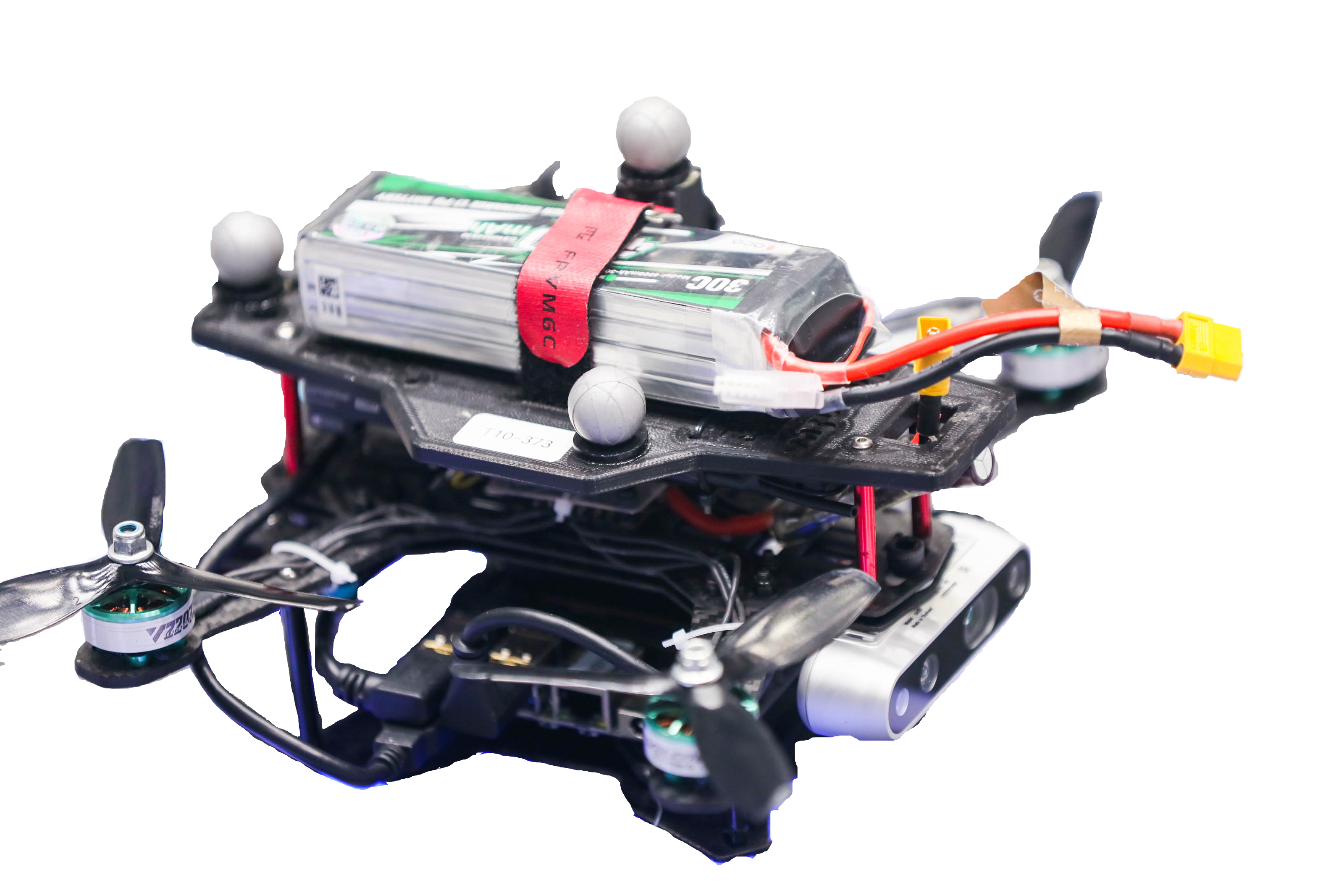}
	\caption{Drone used for real flight.}
	\label{fig:dronepic}
\end{figure}
\subsubsection{Task settings}
The designated flight area is partitioned into three distinct regions, namely the starting area, obstacle area, and goal area. The primary objective of this task for a swarm of four drones is to initiate their flight from the starting area, navigate through the obstacle area without any collision, and ultimately reach the goal area.
\subsubsection{Optimal virtual tube planning}
Firstly, the initial positions ${\bf q}_{0,k}$ of four drones in the start area must be specified, as depicted in Fig. \ref{fig:dronefourground}. It is observed that Drone $1$, Drone $2$, and Drone $3$ are assigned to ${\bf q}_{0,1}$, ${\bf q}_{0,2}$, ${\bf q}_{0,3}$ respectively whose convex hull is the terminal ${\cal C}_0$ in the start area. And the position ${\bf q}_{0,4}$ of Drone $4$ is the affine combination of $\{ {\bf q}_{0,k} \}$ within the terminal ${\cal C}_0$, which is expressed as
\begin{equation}
{{\bf{q}}_{0,4}} = {\theta _{1,4}}{{\bf{q}}_{0,1}} + {\theta _{2,4}}{{\bf{q}}_{0,2}} + {\theta _{3,4}}{{\bf{q}}_{0,3}},\sum\limits_{k = 1}^3 {{\theta _{k,4}}}  = 1.
\end{equation}
\begin{figure}
	\centering
	\includegraphics[width=0.9\linewidth]{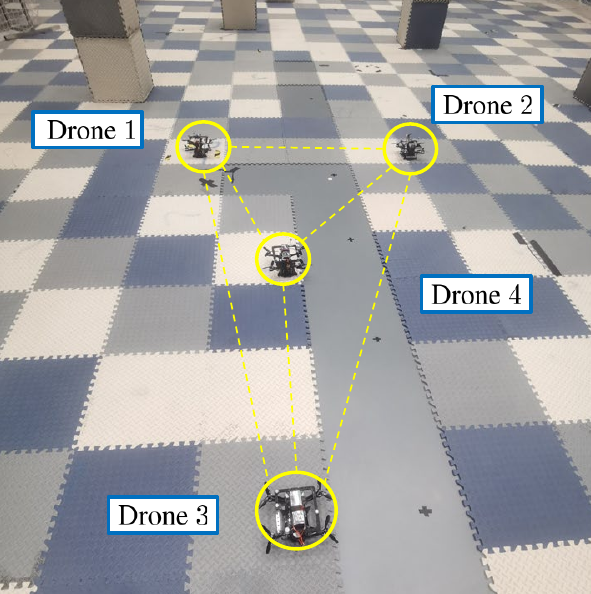}
	\caption{Drone positions at the start area.}
	\label{fig:dronefourground}
\end{figure}
Similarly, the goal positions ${\bf q}_{m,k}$ of drones are assigned. Specially, the goal position ${\bf q}_{m,4}$ is expressed as
\begin{equation}
{{\bf{q}}_{m,4}} = {\theta _{1,4}}{{\bf{q}}_{m,1}} + {\theta _{2,4}}{{\bf{q}}_{m,2}} + {\theta _{3,4}}{{\bf{q}}_{m,3}},\sum\limits_{k = 1}^3 {{\theta _{k,4}}}  = 1.
\end{equation}
Thus, the order pairs $ {\left( {{{\bf{q}}_{0,k}},{{\bf{q}}_{m,k}}} \right)}$ are obtained.

Subsequently, path-finding is used to generate paths for the three drones, as depicted in Fig. \ref{fig:realflightdisplay}(a). For each order pair $ {\left( {{{\bf{q}}_{0,k}},{{\bf{q}}_{m,k}}} \right)}\left(k=1,2,3\right)$, there are 8 intermediate configurations, with $m=7$ denoting the final configuration. To parameterize the trajectories, normalized knots $\{{t_i}\}$ are automatically generated. Three optimization problems are then solved to obtain the optimal trajectories shown in Fig. \ref{fig:realflightdisplay}(b), which are formulated as follows:
\begin{equation}
\begin{array}{l}
{{\bf{h}}^*}\left( {\left( {{{\bf{q}}_{0,1}},{{\bf{q}}_{7,1}}} \right),t} \right) = {\bf{C}}\left( t \right){{\bf{x}}_1},\\
{{\bf{h}}^*}\left( {\left( {{{\bf{q}}_{0,2}},{{\bf{q}}_{7,2}}} \right),t} \right) = {\bf{C}}\left( t \right){{\bf{x}}_2},\\
{{\bf{h}}^*}\left( {\left( {{{\bf{q}}_{0,3}},{{\bf{q}}_{7,3}}} \right),t} \right) = {\bf{C}}\left( t \right){{\bf{x}}_3}.
\end{array}
\end{equation}
Hence, utilizing \emph{Theorem \ref{the:convex_hull}}, the optimal virtual tube $\left( {{{\cal C}_0},{{\cal C}_1},{\bf{f}},{{\bf{h}}^*}} \right)$ is constructed. Since the position ${\bf q}_{0,4}$ is located in the terminal ${\cal C}_0$, the optimal trajectory ${{\bf{h}}^*}\left( {\left( {{{\bf{q}}_{0,4}},{{\bf{q}}_{7,4}}} \right),t} \right)$ for Drone $4$ can be generated by
\begin{equation}
{{\bf{h}}^*}\left( {\left( {{{\bf{q}}_{0,4}},{{\bf{q}}_{7,4}}} \right),t} \right) = {\bf{C}}\left( t \right)\sum\limits_{k = 1}^3 {{\theta _{k,4}}{{\bf{x}}_k}},
\end{equation}
as illustrated in Fig. \ref{fig:realflightdisplay}(c). Upon obtaining the four trajectories for the drones, each drone can track its own trajectory while avoiding the others, as depicted in Fig. \ref{fig:realflightdisplay}(d). As depicted in Fig. \ref{fig:realflight}, the 4th optimal trajectory of Drone $4$, is generated by the convex combination of three optimal trajectories. Subsequently, the drone swarm successfully tracks trajectories to pass through the environment. The velocity of randomly selected Drone $2$ with respect to time is shown in Fig. \ref{fig:drone2}, which indicates that the drone swarm can fast and smoothly pass through the optimal virtual tube.
\begin{figure*}
	\centering
	\includegraphics{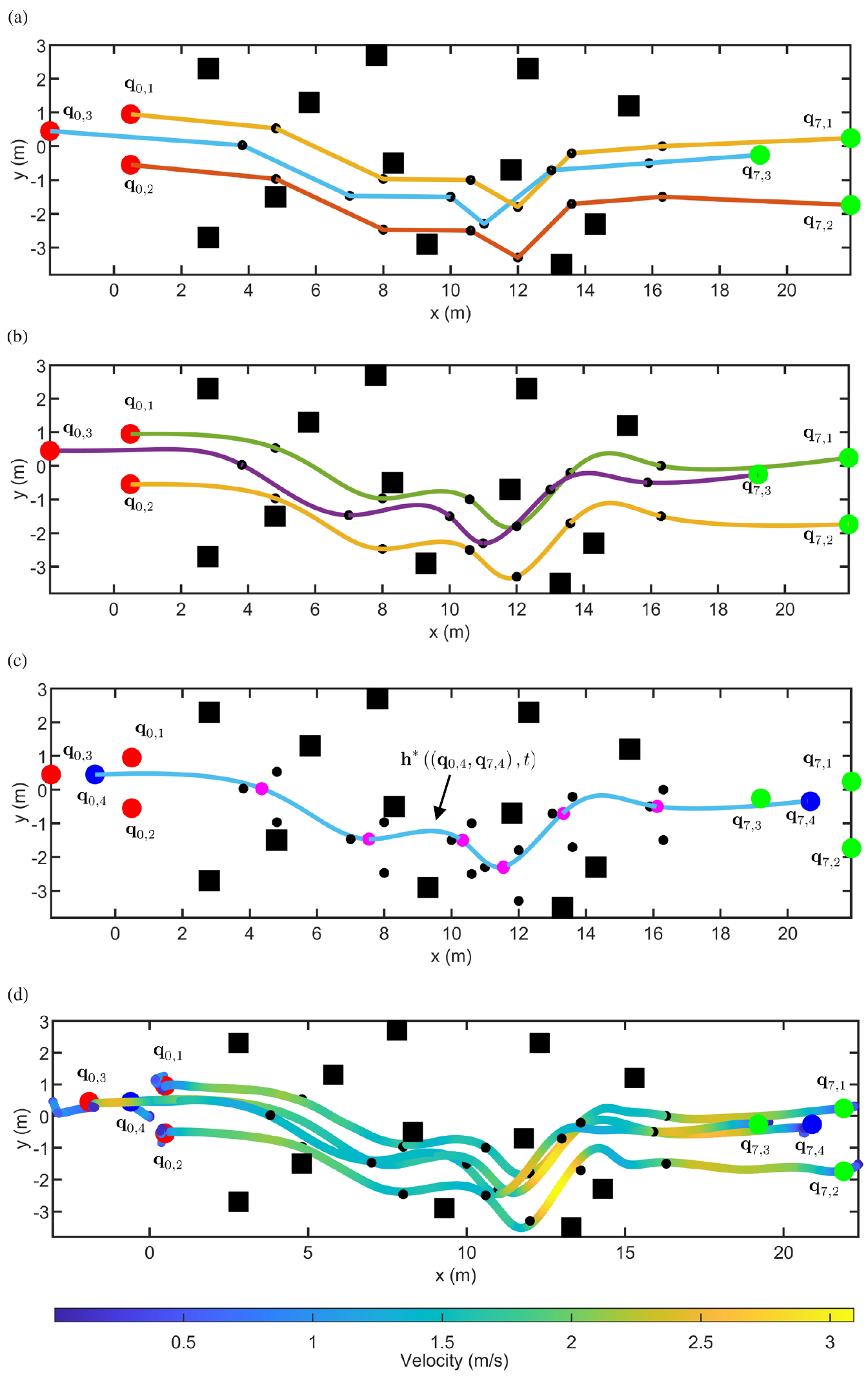}
%
%
%
	\caption{Process of optimal virtual tube planning and application. The squares are obstacles in reality. The red points and green points are start positions and goal positions respectively. (a) The paths for three order pairs. (b) The optimal trajectories for three order pairs. (c) The optimal trajectory of the drone $4$ is generated by interpolation. (d) The true trajectories of drone swarm.}
	\label{fig:realflightdisplay}
\end{figure*}
\begin{figure*}
	\centering
	\includegraphics{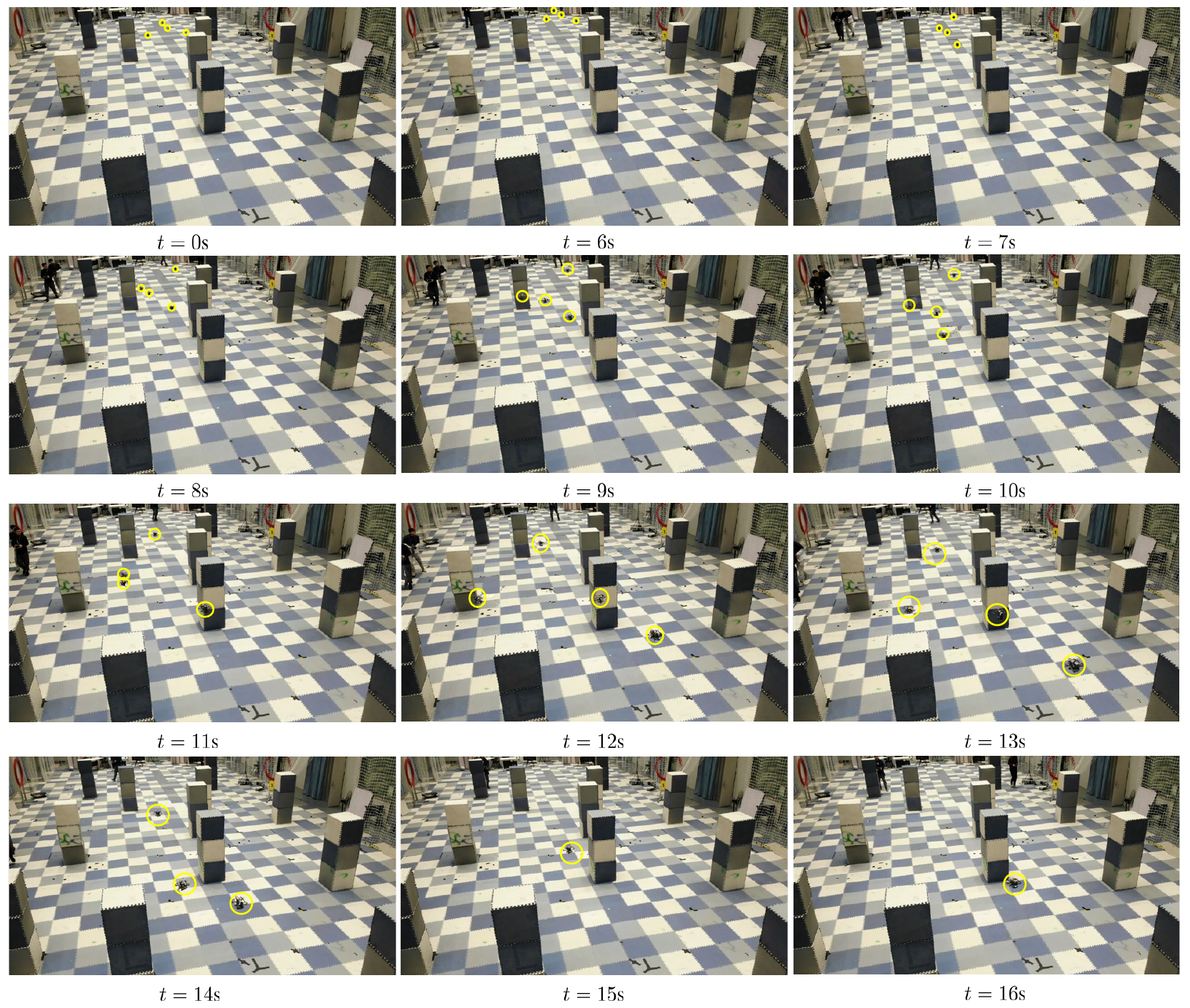}
	\caption{Real flight of drone swarm. }
	\label{fig:realflight}
\end{figure*}
\subsection{Summary}
The hierarchical approach including the optimal virtual tube planning method and model predictive control is validated in simulations and experiments. First, the optimal virtual tube planning method in 2-D and 3-D space is simulated to verify the correctness of \emph{Theorem \ref{the:convex_hull}} by comparing it with the traditional methods. And the MPC controller is implemented in the swarm to achieve inter-drone collision avoidance while the trajectory tracking is achieved. Simulation results demonstrate that the proposed method of optimal virtual tube planning is applicable to large-scale swarm movements. The number of optimization problems is independent of the number of drones when the start positions of drones are within the terminal ${\cal C}_0$. 
Then, comparisons with other methods show that our approach implements the fast and safe movement of the swarm as a group in obstacle-dense environments. In real flight, the approach is validated by the swarm with 4 drones passing through an obstacle-dense environment fast and smoothly.

\begin{figure}
	\centering
	\includegraphics{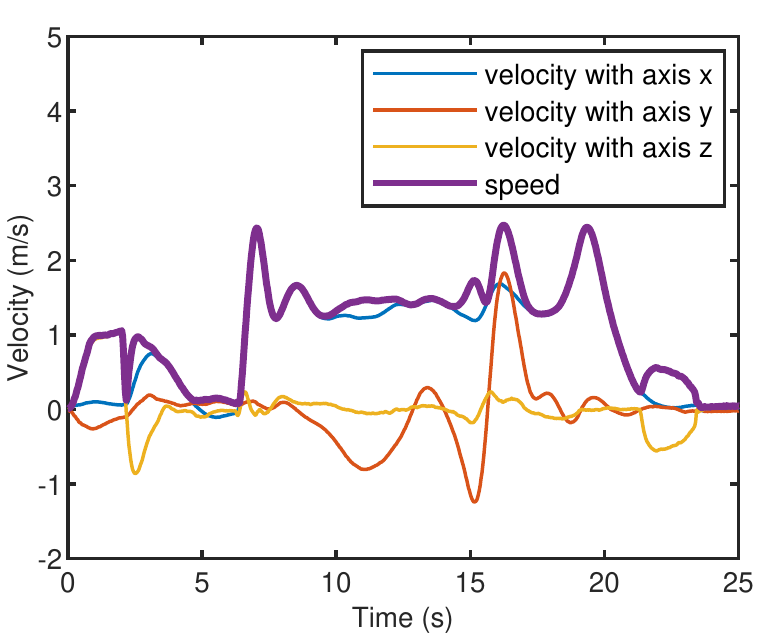}
	\caption{Velocity of Drone $2$.}
	\label{fig:drone2}
\end{figure}


\section{Conclusion and Future Work}
This paper extends the definition of the virtual tube, proposes optimal virtual tubes, and then analyzes their properties. One type of virtual tubes, namely the linear virtual tube with convex hull terminals, is used for constructing an optimal virtual tube. Then, a planning method for this linear virtual tube with convex hull terminals is proposed and its effectiveness is demonstrated in simulations and experiments. The results in 2-D and 3-D show that the error between results obtained through our planning method and optimization is negligible. Therefore, our method can efficiently reduce the computation burden and generate optimal trajectories by convex combinations. An application of the optimal virtual tube, which uses an MPC controller for a drone swarm to track the trajectories in simulations and experiments, shows that the swarm could pass through obstacle-dense environments {fast and safely as a group.}
{In addition, as shown in Fig. \ref{fig:infinitetrajectoriesforsinglerobot}, when the start terminal of the virtual tube is shrunk to a point, it can also be applied to single-robot trajectory planning, that is, to plan infinite optimal trajectories for a single robot to reach the target area.}
\begin{figure}
	\centering
	\includegraphics{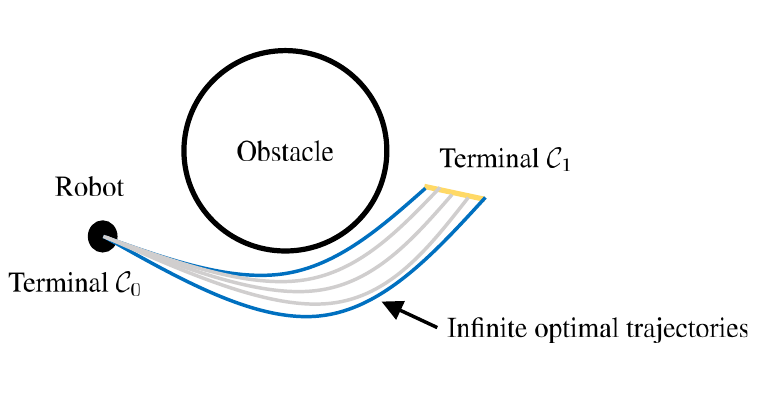}
	\caption{An example for a single robot to generate infinite optimal trajectories. The terminal $\mathcal{C}_0$ is shrunk to be a point where the robot is, and the terminal $\mathcal{C}_1$ is denoted by the yellow line.}
	\label{fig:infinitetrajectoriesforsinglerobot}
\end{figure}
 
{In summary, there are several advantages of our method. (i) The computational burden involved in trajectory planning for large swarm robotics is significantly reduced by our approach, making it practical for real-world applications. (ii) Swarm robotics are enabled by our method to navigate through obstacle-dense environments fast and safely as a cohesive group with desired formations. (iii) Enhanced scalability for swarm robotics is provided by our method, enabling the management of larger swarms while maintaining efficient navigation.}
On the other hand, there are also some limitations of our method. {(i) The large swarm robotics may be blocked in extremely cluttered environments\footnote{{The extremely cluttered environment signifies that the gaps or spaces between obstacles are considerably narrower than the collective volume occupied by the swarm, compared with the obstacle-dense environment. For example, the environment where gaps between obstacles can only hold a single robot to pass through is the extremely cluttered environment.}}.} Due to the property of simply connected cross-sections, all trajectories in the virtual tube are homotopic, preventing the swarm from splitting apart to move around obstacles. This may cause difficulties in extremely cluttered environments. When the space of the gap that the virtual tube passes through is smaller than the physical volume of the swarm, collision avoidance procedure is necessary for the swarm to pass through this gap. However, this may block the swarm. {(ii) The swarm is not allowed dividing into subgroups for a single virtual tube. Multiple virtual tubes become necessary when the swarm is partitioned into distinct subgroups for various goal areas or within extremely cluttered environments. Nonetheless, it is worth noting that managing multiple virtual tubes may entail greater complexity and potentially higher costs compared to individual planning for each robot in extremely cluttered environments.} 

To overcome these limitations and make the swarm autonomous, dependable and affordable (ADA) (\cite{cai2021s}), further studies are included in the following. (i) Path planning for the virtual tube will be designed to reduce the computation time and limit the minimum volume of the cross-section according to the physical volume of the swarm. 
(ii) {The design of the map ${\bf f}$ for ordered pairs needs to be studied further, encompassing not only the convex hull terminals but also the general terminals. This map ${\bf f}$ will aid in the prevention of path crossings and the efficient conservation of energy, among other potential objectives.} In cases where paths for different order pairs intersect, it can be difficult to distinguish between the boundary and interior of the virtual tube, where the boundary is used to constrain the swarm within the virtual tube.
(iii) Furthermore, a method of constructing the virtual tube network or multiple virtual tubes will be explored to allow the swarm to split apart to move around obstacles.

\section{Declaration of Conflicting Interests}
The author(s) declared no potential conflicts of interest with respect to the research, authorship, and/or publication of this article.
\section{Funding}
This work was supported by the National Natural Science Foundation of China under Grant 61973015.


\begin{appendices}
	\section{A Proofs of Lemma 1 and Lemma 2}
	\subsection{A.1 Proof of Lemma 1}
	\begin{proof}
		In \emph{Step 1}, it is proved that ${\bf{x}}({\bm{\theta}})$ is feasible. Then the optimality of ${\bf{x}}({\bm{\theta}})$ is shown in \emph{Step 2}.
		
		\textbf{Step 1}. According to an optimality criterion (\cite{Boyd_convex_2004}) for differentiable $f_0$. Let $\mathcal{X}_{\theta}$, $\mathcal{X}_{k}$ denote feasible sets:
		\[{{\mathcal{X}}_{\theta}} = \left\{ {{\bf{x}}|{\bf{Ax}} = {{\bf{b}}\left({\bm{\theta}}\right)}} \right\},{{\mathcal{X}}_{k}} = \left\{ {{\bf{x}}|{\bf{Ax}} = {{\bf{b}}_k}} \right\}.\]
		Then for all ${\mathbf{x}} \in {{{\mathcal{X}}_{k}}} $,
		\begin{equation}
		\nabla f_0\left( {{{\bf{x}}_k}} \right)^{\rm{T}}\left( {{\bf{x}} - {{\bf{x}}_k}} \right) \ge 0,k=1,2,...,q.
		\label{equ:optimal_condition_std}
		\end{equation}
		Since ${\bf{x}}_k$ is feasible, each solution ${\bf{x}}$ in Equation (\ref{equ:optimal_condition_std}) can be expressed as the combination of a special solution and a solution in zero space, such as ${\bf{x}} = {{\bf{x}}_k} + {\bf{v}}$,${\bf{v}} \in N\left( {\bf{A}} \right)$.
		Therefore, this optimal condition (\ref{equ:optimal_condition_std}) is transformed into
		\begin{equation}
		\nabla {f_0}{\left( {{{\bf{x}}_k}} \right)^{\rm{T}}}{\bf{v}} \ge {\bf{0}}, {\rm{for \, all \, }}{\bf{v}}\in  N\left( {\bf{A}} \right).
		\label{equ:optimal_condition_for_x0_std}
		\end{equation}
		
		When ${\bf{b}} = {\bf{b}}({\bm{\theta}})$, we substitute ${\bf{x}}\left( {\bm{\theta}} \right)$ into (\ref{equ:optim_problem}) to obtain 
		\[
		{\bf{Ax}}\left( {\bm{\theta}} \right)   = \sum\limits_{k = 1}^q {{\theta _k}{\bf{A}}{{\bf{x}}_k}}
		= \sum\limits_{k = 1}^q {{\theta _k}{{\bf{b}}_k}}\\
		= {\bf{b}}\left( {\bm{\theta}} \right)
		.\]
		Therefore, ${{\bf{x}}({\bm{\theta}})}$ is feasible.
		
		\textbf{Step 2}. For all ${\mathbf{x}} \in {{{\mathcal{X}}_\theta}} $, it has
		\begin{equation}
		\nabla f_0\left( {{\bf{x}}\left( {\bm{\theta}} \right)} \right)^{\rm{T}}\left( {{\bf{x}} - {\bf{x}}\left( {\bm{\theta}} \right)} \right)
		={\sum\limits_{k = 1}^q {{\theta _k}\nabla {f_0}\left( {{{\bf{x}}_k}} \right)} ^{\rm{T}}}{\bf{v}}.
		\label{equ:optimal_condition_for_xm_std}
		\end{equation}
		Combining Equation (\ref{equ:optimal_condition_for_xm_std}) with (\ref{equ:optimal_condition_for_x0_std}) derives 
		\begin{equation}
		\nabla f_0\left( {{\bf{x}}\left( {\bm{\theta}} \right)} \right)^{\rm{T}}\left( {{\bf{x}} - {\bf{x}}\left({\bm{\theta}} \right)} \right) \ge 0, {\rm{for \, all \,}} {\bf{x}} \in {{{\mathcal{X}}_\theta}}.
		\end{equation}
		Therefore, ${\bf{x}}({\bm{\theta}})$ is optimal. $\hfill\blacksquare$ 
	\end{proof}

	\subsection{A.2 Proof of Lemma 2}
	\begin{proof}
		In \emph{Step 1}, it is shown that ${\bf{x}}({\bm{\theta}})$ is feasible. Then the optimality of ${\bf{x}}({\bm{\theta}})$ is shown in \emph{Step 2}.
		
		\textbf{Step 1}. According to an optimality criterion (\cite{Boyd_convex_2004}) for differentiable $f_0$. Let $\mathcal{X}_{\theta}$, $\mathcal{X}_{k}$ denote feasible sets,
		\[{{\mathcal{X}}_\theta} = \left\{ {{\bf{x}}|{\bf{Ax}} = {{\bf{b}}{\left({\bm{\theta}}\right)}}}, {f_i}\left( {\bf{x}} \right)  \le 0, i=1,...,n_c \right\},\]
		\[{{\mathcal{X}}_{k}} = \left\{ {{\bf{x}}|{\bf{Ax}} = {{\bf{b}}_k}} , {f_i}\left( {\bf{x}} \right)  \le 0, i=1,...,n_c\right\}.\]
		Let ${\mathcal{V}}_k$ denote the equivalent feasible set,
		\[{{\mathcal{V}}_k} = \left\{ { {\bf{v}}| {\bf{A}}({\bf{x}}_k+{\bf{v}}) = {{\bf{b}}_k},{f_i}\left( {\bf{x}}_k  + {\bf{v}}\right) \le 0,i = 1,...,n_c} \right\}.\]
		And for all ${\mathbf{x}} \in {{{\mathcal{X}}_k}} $,
		\begin{equation}
		\nabla f_0\left( {{{\bf{x}}_k}} \right)^{\rm{T}}\left( {{\bf{x}} - {{\bf{x}}_k}} \right) \ge 0.
		\label{equ:optimal_condition}
		\end{equation}
		Since ${\bf{x}}_k$ is feasible, every  ${\bf{x}}$ has ${\bf{x}} = {{\bf{x}}_k} + {\bf{v}}$, ${\bf{v}} \in {\mathcal{V}}_k$.
		Therefore, the optimal condition (\ref{equ:optimal_condition}) can be expressed as:
		\begin{equation}
		\nabla {f_0}{\left( {{{\bf{x}}_k}} \right)^{\rm{T}}}{\bf{v}} \ge 0,  \, for \, all \, {\bf{v}} \in {\mathcal{V}}_k.
		\label{equ:optimal_condition_for_x0}
		\end{equation}
		
		Let ${\bf{b}}$ be ${\bf{b}}\left({\bm{\theta}}\right)$. Substituting ${\bf{x}}\left( {\bm{\theta}} \right)$ into constraints yields
		\[
		{\bf{Ax}}\left( {\bm{\theta}} \right)   = \sum\limits_{k = 1}^q {{\theta _k}{\bf{A}}{{\bf{x}}_k}}
		= \sum\limits_{k = 1}^q {{\theta _k}{{\bf{b}}_k}}\\
		= {\bf{b}}\left( {\bm{\theta}} \right)
		,\]
		\[{f_j}\left( {{\bf{x}}\left( {{\bm\theta}} \right)} \right) \le \sum\limits_{k = 1}^q {{\theta _k}{f_j}\left( {{{\bf{x}}_k}} \right)}  \le 0,j = 1,...,n_c.\]
		Therefore, ${{\bf{x}}\left({\bm{\theta}}\right)}$ is feasible. 
		
		\textbf{Step 2}. For all ${\mathbf{x}} \in {{{\mathcal{X}}_\theta}} $, it has equivalent constraint ${{\mathcal{V}}_{\theta}} = \sum\nolimits_{k = 1}^q {{\theta _k}{{\mathcal{V}}_k}} $. Therefore, for all ${\bf{v}}_\theta \in {{\mathcal{V}}_{\theta}}$, there exist specific ${\bf{v}}_k \in {{\mathcal{V}}_k}$ so that 
		\[{\bf{v}}_\theta = \sum\limits_{k = 1}^q {{\theta _k}{{\bf{v}}_k}}. \]
		In the following, we will show  $\nabla {f_0}{\left( {{\bf{x}}\left( {{{\bm{\theta}}}} \right)} \right)^{\rm{T}}}{{\bf{v}}_\theta } \ge 0$.
		
		When $q=2$, 
		\begin{equation}
		\begin{array}{*{20}{l}}
		{}&{\nabla {f_0}{{\left( {{\bf{x}}\left( {\bm \theta } \right)} \right)}^{\rm{T}}}{{\bf{v}}_\theta }}\\
		{}&{ = \left( {\theta _1^2 + {\theta _1}{\theta _2}} \right)\nabla {f_0}{{\left( {{{\bf{x}}_1}} \right)}^{\rm{T}}}{{\bf{v}}_1} + \left( {\theta _2^2 + {\theta _1}{\theta _2}} \right)\nabla {f_0}{{\left( {{{\bf{x}}_2}} \right)}^{\rm{T}}}{{\bf{v}}_2}}\\
		{}&{ + {\theta _1}{\theta _2}{{\left( {\nabla {f_0}\left( {{{\bf{x}}_1}} \right) - \nabla {f_0}\left( {{{\bf{x}}_2}} \right)} \right)}^{\rm{T}}}\left( {{{\bf{x}}_1} - {{\bf{x}}_2}} \right)}.
		\end{array}
		\label{equ:optimal_condition_for_xm}
		\end{equation} 
		Since ${f_0}$ is a convex function, there have
		\begin{subequations}
			\begin{align}
			{f_0}\left( {{{\bf{x}}_1}} \right) \ge {f_0}\left( {{{\bf{x}}_2}} \right) + \nabla {f_0}\left( {{{\bf{x}}_2}} \right)^{\rm{T}}\left( {{{\bf{x}}_1} - {{\bf{x}}_2}} \right),\label{equ:math_indu_1}\\
			{f_0}\left( {{{\bf{x}}_2}} \right) \ge {f_0}\left( {{{\bf{x}}_1}} \right) + \nabla {f_0}\left( {{{\bf{x}}_1}} \right)^{\rm{T}}\left( {{{\bf{x}}_2} - {{\bf{x}}_1}} \right).\label{equ:math_indu_2}
			\end{align}
		\end{subequations}
		Combining (\ref{equ:math_indu_1}) and (\ref{equ:math_indu_2}), we obtain
		\begin{equation}
		\begin{array}{c}
		{\left( {\nabla {f_0}\left( {{{\bf{x}}_1}} \right) - \nabla {f_0}\left( {{{\bf{x}}_2}} \right)} \right)^{\rm{T}}}\left( {{{\bf{x}}_1} - {{\bf{x}}_2}} \right) \ge 0.
		\end{array}
		\label{equ:convex_condition}
		\end{equation}
		Combining Equation (\ref{equ:optimal_condition_for_xm}) with (\ref{equ:convex_condition}) results in
		\begin{equation}
		\nabla f_0\left( {{\bf{x}}\left( {\bm{\theta}} \right)} \right)^{\rm{T}} { {\bf{v}}_{\theta}} \ge 0, \, for \, all \, {\bf{v}}_\theta \in {{{\mathcal{V}}_\theta}}.
		\label{equ:optimal-condition-eq}
		\end{equation}
		
		Suppose that Equation (\ref{equ:optimal-condition-eq}) holds when $q=n$. Then, when $q=n+1$, we have
		\begin{equation}
		\begin{array}{ll}
		{{\bf{v}}_\theta } &  = {\theta _{n + 1}}{{\bf{v}}_{n + 1}} + \sum\nolimits_{k = 1}^n {{\theta _k}{{\bf{v}}_k}} \\
		&  = {\theta _{n + 1}}{{\bf{v}}_{n + 1}} + \sum\nolimits_{k = 1}^n {{\theta _k}} \sum\nolimits_{k = 1}^n {\frac{{{\theta _k}}}{{\sum\nolimits_{k = 1}^n {{\theta _k}} }}{{\bf{v}}_k}} \\
		&  = {\theta _{n + 1}}{{\bf{v}}_{n + 1}} + \left( {1 - {\theta _{n + 1}}} \right)\sum\nolimits_{k = 1}^n {{{\tilde \theta }_k}{{\bf{v}}_k}}, \\
		{{\bf{x}}\left({\bm \theta} \right)} &  = {\theta _{n + 1}}{{\bf{x}}_{n + 1}} + \left( {1 - {\theta _{n + 1}}} \right)\sum\nolimits_{k = 1}^n {{{\tilde \theta }_k}{{\bf{x}}_k}}. 
		\end{array}
		\label{equ:n_1}
		\end{equation}
		
		We substitute (\ref{equ:n_1}) into (\ref{equ:optimal-condition-eq}) to obtain that 
		\[\begin{array}{l}
		\nabla {f_0}{\left( {{\bf{x}}\left( {\bm{\theta }} \right)} \right)^{\rm{T}}}{{\bf{v}}_\theta } = \theta _{n + 1}^2\nabla {f_0}{\left( {{{\bf{x}}_{n + 1}}} \right)^{\rm{T}}}{{\bf{v}}_{n + 1}}\\
		+ {\theta _{n + 1}}\left( {1 - {\theta _{n + 1}}} \right)\nabla {f_0}{\left( {{{\bf{x}}_{n + 1}}} \right)^{\rm{T}}}{{\bf{v}}_{n + 1}}\\
		+ {\left( {1 - {\theta _{n + 1}}} \right)^2}\nabla {f_0}{\left( {\sum\limits_{k = 1}^n {{{\tilde \theta }_k}{{\bf{x}}_k}} } \right)^{\rm{T}}}\sum\limits_{k = 1}^n {{{\tilde \theta }_k}{{\bf{v}}_k}} \\
		+ {\theta _{n + 1}}\left( {1 - {\theta _{n + 1}}} \right)\nabla {f_0}{\left( {\sum\limits_{k = 1}^n {{{\tilde \theta }_k}{{\bf{x}}_k}} } \right)^{\rm{T}}}\sum\limits_{k = 1}^n {{{\tilde \theta }_k}{{\bf{v}}_k}} \\
		+ {\theta _{n + 1}}\left( {1 - {\theta _{n + 1}}} \right){\left( {\nabla {f_0}\left( {{{\bf{x}}_{n + 1}}} \right) - \nabla {f_0}\left( {\sum\limits_{k = 1}^n {{{\tilde \theta }_k}{{\bf{x}}_k}} } \right)} \right)^{\rm{T}}}\\
		\cdot \left( {{{\bf{x}}_{n + 1}} - \sum\limits_{k = 1}^n {{{\tilde \theta }_k}{{\bf{x}}_k}} } \right).
		\end{array}\]

		In the same way, we could obtain
		\[\begin{array}{l}
		{\left( {\nabla {f_0}\left( {{{\bf{x}}_{n + 1}}} \right) - \nabla {f_0}\left( {\sum\limits_{k = 1}^n {{{\tilde \theta }_k}{{\bf{x}}_k}} } \right)} \right)^{\rm{T}}}\left( {{{\bf{x}}_{n + 1}} - \sum\limits_{k = 1}^n {{{\tilde \theta }_k}{{\bf{x}}_k}} } \right)\\
		\ge 0.
		\end{array}\]
		It is derived that
		\begin{equation}
		\nabla f_0\left( {{\bf{x}}\left( {\bm{\theta}} \right)} \right)^{\rm{T}} { {\bf{v}}_{\theta}} \ge 0, \, for \, all \, {\bf{v}}_\theta \in {{{\mathcal{V}}_\theta}}.
		\end{equation}
		Therefore, ${\bf{x}}({\bm{\theta}})$ is optimal. $\hfill\blacksquare$
	\end{proof}
\end{appendices}

\end{document}